\documentclass{article}

\usepackage[utf8]{inputenc} 
\usepackage[T1]{fontenc}    
\usepackage{microtype}
\usepackage{graphicx}
\usepackage{caption}
\usepackage{subcaption}
\usepackage{booktabs} 
\usepackage{microtype}      

\usepackage{algorithm}
\usepackage{algorithmic}
\usepackage{wrapfig}
\usepackage{amsfonts}       
\usepackage{nicefrac}       
\usepackage[round]{natbib}  
\usepackage{fullpage}
\usepackage{authblk}
\usepackage{hyperref}       
\hypersetup{colorlinks, citecolor={blue}}

\renewcommand{\algorithmiccomment}[1]{{\color{blue}\bgroup\hfill//~#1\egroup}}

\usepackage{enumitem}
\usepackage{multirow}
\usepackage{xspace}

\usepackage{Definitions}
\usepackage{color}

\renewcommand{\le}{\leqslant}

\renewcommand{\ge}{\geqslant}

\newcommand{\estname}{{CoinDICE}\xspace}
\newcommand{\qlp}{$Q$-LP\xspace}
\newcommand{\1}[1]{\mathbf{1}\left\{#1\right\}}
\newcommand{\repeatthm}[2]{\textbf{Theorem~\ref{#1}} \textit{#2}}
\newcommand{\linspan}[1]{\operatorname{span}\left\{#1\right\}}

\def\bellman{\mathcal{B}_{\pi}}

\def\confrange{\chi_{(1)}^{2, 1-\alpha}}
\def\confrange{\xi}
\def\rmax{R_{\mathrm{max}}}
\def\Qbeta{Q_\beta}
\def\rholag{\hat{\rho}_\pi}

\title{\estname:~Off-Policy~Confidence~Interval~Estimation}

\author{
   $^*$Bo Dai$^1$,  \thanks{Equal contribution. Email: \texttt{\{bodai, ofirnachum\}@google.com}.} Ofir Nachum$^1$,  Yinlam Chow$^1$\\\vspace{-2mm}
   Lihong Li$^1$, Csaba Szepesv\'{a}ri$^{2,3}$,  Dale Schuurmans$^{1,3}$\\ \vspace{3mm}
   $^1$Google Research\quad $^2$University of Alberta\quad  $^3$DeepMind
}
\date{}

\begin{document}

\maketitle
\begin{abstract}
We study \emph{high-confidence behavior-agnostic off-policy evaluation} in reinforcement learning, where the goal is to estimate a confidence interval on a target policy's value, given only access to a static experience dataset collected by unknown behavior policies. Starting from a function space embedding of the linear program formulation of the $Q$-function, we obtain an optimization problem with generalized estimating equation constraints. By applying the generalized empirical likelihood method to the resulting Lagrangian, we propose \emph{\estname}, a novel and efficient algorithm for computing confidence intervals. Theoretically, we prove the obtained confidence intervals are valid, in both asymptotic and finite-sample regimes. Empirically, we show in a variety of benchmarks that the confidence interval estimates are tighter and more accurate than existing methods.\footnote{Open-source code for \estname is available at \href{https://github.com/google-research/dice_rl}{https://github.com/google-research/dice\_rl}.}
\end{abstract}

\section{Introduction}\label{sec:intro}

One of the major barriers that hinders the application of reinforcement learning~(RL) is the ability to evaluate new policies reliably \emph{before} deployment, a problem generally known as \emph{off-policy evaluation}~(OPE). In many real-world domains, \eg, healthcare~\citep{MurVanRobetal01,gottesman18evaluating}, recommendation~\citep{LiChuLanWan11,CheBeuCovSagetal19}, and education~\citep{ManLiuLevBruetal14}, deploying a new policy can be expensive, risky or unsafe. Accordingly, OPE has seen a recent resurgence of research interest, with many methods proposed to estimate the value of a policy~\citep{PreSutSin00,dudik2011doubly,BotPetQuiChaetal13,jiang16doubly,thomas16data,LiuLiTanZho18,NacChoDaiLi19,KalUeh19,kallus19intrinsically,zhang20gradientdice}.

However, the very settings where OPE is necessary usually entail limited data access. In these cases, obtaining knowledge of the uncertainty of the estimate is as important as having a consistent estimator. That is, rather than a \emph{point estimate}, many applications would benefit significantly from having \emph{confidence intervals} on the value of a policy. The problem of estimating these confidence intervals, known as \emph{high-confidence off-policy evaluation} (HCOPE)~\citep{thomas2015high}, is imperative in real-world decision making, where deploying a policy without high-probability safety guarantees can have catastrophic consequences \citep{thomas2015safe}. Most existing high-confidence off-policy evaluation algorithms in RL~\citep{BotPetQuiChaetal13,thomas2015high_opt,thomas2015high,hanna17bootstrapping} construct such intervals using statistical techniques such as concentration inequalities and the bootstrap applied to importance corrected estimates of policy value. The primary challenge with these correction-based approaches is the high variance resulting from multiplying per-step importance ratios in long-horizon problems. Moreover, they typically require full knowledge (or a good estimate) of the behavior policy, which is not easily available in behavior-agnostic OPE settings~\citep{NacChoDaiLi19}.

In this work, we propose an algorithm for behavior-agnostic HCOPE.  We start from a linear programming formulation of the state-action value function. We show that the value of the policy may be obtained from a Lagrangian optimization problem for generalized estimating equations over data sampled from off-policy distributions. This observation inspires a generalized empirical likelihood approach~\citep{Owen01,BroKez12,DucGlyNam16} to confidence interval estimation. These derivations enable us to express high-confidence lower and upper bounds for the policy value as results of minimax optimizations over an arbitrary offline dataset, with the appropriate distribution corrections being implicitly estimated during the optimization. åWe translate this understanding into a practical estimator, \emph{Confidence Interval DIstribution Correction Estimation} (\estname), and design an efficient algorithm for implementing it. We then justify the asymptotic coverage of these bounds and present non-asymptotic guarantees to characterize finite-sample effects. Notably, \estname is behavior-agnostic and its objective function does not involve any per-step importance ratios, and so the estimator is less susceptible to high-variance gradient updates. We evaluate \estname in a number of settings and show that it provides both tighter confidence interval estimates and more correctly matches the desired statistical coverage compared to existing methods.

\vspace{-3mm}
\section{Preliminaries}\label{sec:prelim}
\vspace{-2mm}

For a set $W$, the set of probability measures over $W$ is denoted by
$\Pcal\rbr{W}$.\footnote{All sets and maps are assumed to satisfy appropriate measurability conditions; which we will omit from below for the sake of reducing clutter.} We consider a Markov Decision Process (MDP)~\citep{Puterman14}, $\Mcal = \rbr{S, A, T, R, \gamma, \mu_0}$, where $S$ denotes the state space, $A$ denotes the action space, $T: S \times A \rightarrow \Pcal\rbr{S}$ is the transition probability kernel, $R: S\times A\rightarrow \Pcal\rbr{[0, \rmax]}$ is a bounded reward kernel, $\gamma\in (0, 1]$ is the discount factor, and $\mu_0$ is the initial state distribution.

A policy, $\pi: S\rightarrow \Pcal\rbr{A}$, can be used to generate a random trajectory by starting from $s_0\sim \mu_0\rbr{s}$, then following $a_t\sim\pi\rbr{s_t}$, $r_t\sim R\rbr{s_t, a_t}$ and $s_{t+1}\sim T\rbr{s_t, a_t}$ for $t\ge0$. The state- and action-value functions of $\pi$ are denoted  $V^\pi$ and $Q^\pi$, respectively. The policy also induces an occupancy measure, $d^\pi(s,a) \defeq (1-\gamma) \EE_\pi\sbr{\sum_{t\ge0}\gamma^t\1{s_t=s,a_t=a}}$, the normalized discounted probability of visiting $(s,a)$ in a trajectory generated by $\pi$, where $\1{\cdot}$ is the indicator function. Finally, the \emph{policy value} is defined as the \emph{normalized} expected reward accumulated along a trajectory:
\begin{equation}
\rho_\pi \defeq \rbr{1 - \gamma}\EE\sbr{\sum_{t=0}^\infty \gamma^t r_t| s_0\sim\mu_0, a_t\sim\pi\rbr{s_t}, r_t\sim R\rbr{s_t, a_t}, s_{t+1}\sim T\rbr{s_t, a_t}}.
\end{equation} 
We are interested in estimating the policy value and its confidence interval~(CI) in the \emph{behavior agnostic off-policy} setting~\citep{NacChoDaiLi19,ZhaDaiLiSch20}, where interaction with the environment is limited to a static dataset of experience $\Dcal\defeq \cbr{\rbr{s, a, s', r}_i}_{i=1}^n$. Each tuple in $\Dcal$ is generated according to 
$
\rbr{s, a}\sim d^{\Dcal}, r\sim R\rbr{s, a} \mbox{ and } s'\sim T\rbr{s, a},
$
where $d^{\Dcal}$ is an unknown distribution over $S\times A$, perhaps induced by one or more unknown behavior policies. The initial distribution $\mu_0\rbr{s}$ is assumed to be easy to sample from, as is typical in practice. Abusing notation, we denote by $d^{\Dcal}$ both the distribution over $\rbr{s, a, s', r}$ and its marginal on $\rbr{s, a}$. We use $\EE_d\sbr{\cdot}$ for the expectation over a given distribution $d$, and $\EE_{\Dcal}\sbr{\cdot}$ for its empirical approximation using
$\Dcal$.

Following previous work~\citep{SutSzeGerBow12,UehHuaJia19,ZhaDaiLiSch20}, for ease of exposition we assume the transitions in $\Dcal$ are \iid. However, our results may be extended to fast-mixing, ergodic MDPs, where the the empirical distribution of states along a long trajectory is close to being \iid~\citep{AntSzeMun08b,LazGhaMun12,DaiShaLiXiaHeetal17,DucGlyNam16}.

Under mild regularity assumptions, the OPE problem may be formulated as a linear program -- referred to as the \qlp~\citep{NacDaiKosChoetal19,NacDai20} -- with the following primal and dual forms: \\
\vspace{1mm}
\resizebox{0.4\textwidth}{!}
{
\begin{minipage}{0.45\textwidth}
\begin{align}\label{eq:q_lp}
\min_{Q: S\times A\rightarrow \RR}& \rbr{1-\gamma}\EE_{\mu_0\pi}\sbr{Q\rbr{s_0, a_0}} \\
  \st~~~ &Q\rbr{s, a}\ge R\rbr{s, a} + \gamma\cdot \Pcal^\pi Q\rbr{s, a},\nonumber \\
  & \forall \rbr{s, a}\in S\times A,\nonumber
\end{align}
\end{minipage}
}
~~and~~
\resizebox{0.5\textwidth}{!}
{
\begin{minipage}{0.45\textwidth}
\begin{align} \label{eq:rho_lp}
\max_{d:S\times A\rightarrow \RR_+} &\EE_{d}\sbr{r\rbr{s, a}}\\
  \hspace{-3mm}\st~~~ &d\rbr{s, a} = \hspace{-0mm}\rbr{1 - \gamma}\mu_0\pi\rbr{s, a}\hspace{-0mm} +\hspace{-0mm} \gamma\cdot\Pcal^\pi_* d\rbr{s, a},\nonumber \\
  & \forall \rbr{s, a}\in S\times A,\nonumber 
\end{align}
\end{minipage}
}
where the operator $\Pcal^\pi$ and its adjoint, $\Pcal^\pi_*$, are defined as
\begin{align*}
\Pcal^\pi Q\rbr{s, a} &\defeq \EE_{s'\sim T\rbr{\cdot|s, a}, a'\sim \pi\rbr{\cdot|s'}}\sbr{Q\rbr{s', a'}}\,,\\ 
\Pcal^\pi_* d\rbr{s, a} &\defeq \pi\rbr{a|s}\sum_{\stil,\atil}T\rbr{s|\stil,\atil}d\rbr{\stil, \atil} \,.
\end{align*}
The optimal solutions of~\eqref{eq:q_lp} and~\eqref{eq:rho_lp} are the $Q$-function, $Q^\pi$, and stationary state-action occupancy, $d^\pi$, respectively, for policy $\pi$; see \citet[Theorems~3 \& 5]{NacDaiKosChoetal19} for details as well as extensions to the undiscounted case.

Using the Lagrangian of \eqref{eq:q_lp} or~\eqref{eq:rho_lp}, we have
\begin{equation}\label{eq:dice_lagrangian}
\resizebox{0.93\textwidth}{!}
{
$
\rho_\pi =\min_{Q}\max_{\tau\ge0}\,\, \rbr{1 -\gamma}\EE_{\mu_0\pi}\sbr{Q\rbr{s_0, a_0}} + \EE_{d^\Dcal}\sbr{\tau\rbr{s, a}\rbr{R\rbr{s, a} + \gamma Q\rbr{s', a'} - Q\rbr{s, a}}},
$
}
\end{equation}
where $\tau\rbr{s, a}\!\defeq\!\frac{d\rbr{s, a}}{d^{\Dcal}\rbr{s, a}}$ is the \emph{stationary distribution corrector}. One of the key benefits of the minimax optimization~\eqref{eq:dice_lagrangian} is that both expectations can be immediately approximated by sample averages.\footnote{We assume one can sample initial states from $\mu_0$, an assumption that often holds in practice. Then, the data in $\Dcal$ can be treated as being augmented as $\rbr{s_0, a_0, s, a, r, s', a'}$ with $a_0\sim \pi\rbr{a|s_0}, a'\sim\pi\rbr{a|s'}$. 
}
In fact, this formulation allows the derivation of several recent behavior-agnostic OPE estimators in a unified manner~\citep{NacChoDaiLi19,UehHuaJia19,ZhaDaiLiSch20,NacDai20}.

\section{\estname}\label{sec:method}

We now develop a new approach to obtaining confidence intervals for OPE. The algorithm, \emph{COnfidence INterval stationary DIstribution Correction Estimation~(\estname)}, is derived by combining function space embedding and the previously described \qlp.

\subsection{Function Space Embedding of Constraints}\label{sec:lp_embedding}

Both the primal and dual forms of the \qlp contain $\abr{S}\abr{A}$ constraints that involve expectations over state transition probabilities. Working directly with these constraints quickly becomes computationally and statistically prohibitive when $\abr{S}\abr{A}$ is large or infinite, as with standard LP approaches~\citep{FarVan03}. Instead, we consider a relaxation that embeds the constraints in a function space:
\begin{align}\label{eq:dual_lp_embedding}
\tilde\rho_{\pi}\defeq \max_{d:S\times A\rightarrow \RR_+} \EE_{d}\sbr{r\rbr{s, a}}\quad
\st\,\,\inner{\phi}{d} = \inner{\phi}{\rbr{1 - \gamma}\mu_0\pi + \gamma\cdot\Pcal^\pi_* d }, 
\end{align}
where $\phi:S\times A\rightarrow \Omega^p\subset\RR^p$ is a feature map, and $\inner{\phi}{d}\defeq \int \phi\rbr{s, a}d\rbr{s, a}ds da$.  By projecting the constraints onto a function space with feature mapping $\phi$, we can reduce the number of constraints from $\abr{S}\abr{A}$ to $p$.  Note that $p$ may still be infinite. The constraint in~\eqref{eq:dual_lp_embedding} can be written as \emph{generalized estimating equations}~\citep{QinLaw94,LamZhou17} for the correction ratio $\tau\rbr{s, a}$ over augmented samples $x\defeq \rbr{s_0, a_0, s, a, r, s',a'}$ with $\rbr{s_0, a_0}\sim \mu_0\pi$, $\rbr{s, a, r, s'}\sim d^{\Dcal}$, and $a' \sim \pi(\cdot|s')$,
\begin{equation}
\inner{\phi}{d} = \inner{\phi}{\rbr{1 - \gamma}\mu_0\pi + \gamma\cdot\Pcal^\pi_* d }
\,\,\,\,\Leftrightarrow\,\,\,\, \EE_{x}\sbr{\Delta\rbr{x; \tau, \phi}} = 0,
\end{equation}
where $\Delta\rbr{x; \tau, \phi}\defeq \rbr{1-\gamma}\phi\rbr{s_0, a_0} + \tau\rbr{s, a}\rbr{\gamma\phi\rbr{s',a'} - \phi\rbr{s, a}}$. 
The corresponding Lagrangian is
\begin{equation}\label{eq:dual_lagrangian_embedding}
\tilde\rho_{\pi} = \max_{\tau:S\times A\rightarrow \RR_+} \min_{\beta\in\RR^p} \EE_{d^\Dcal}\sbr{\tau\cdot r\rbr{s, a}} + \inner{\beta}{\EE_{d^\Dcal}\sbr{\Delta\rbr{x; \tau, \phi}}}\,.
\end{equation}

This embedding approach for the dual \qlp is closely related to approximation methods for the standard state-value LP~\citep{FarVan03,PazPar11,LakBhaSze17}. The gap between the solutions to~\eqref{eq:dual_lp_embedding} and the original dual LP~\eqref{eq:rho_lp} depends on the expressiveness of the feature mapping $\phi$. Before stating a theorem that quantifies the  error, we first offer a few examples to provide intuition for the role played by $\phi$.

\vspace{-2mm}
\paragraph{Example (Indicator functions):} Suppose $p=\abr{S}\abr{A}$ is finite and $\phi = [\delta_{{s, a}}]_{(s, a)\in S\times A}$, where $\delta_{{s, a}} \in \cbr{0, 1}^p$ with $\delta_{s, a}=1$ at position $\rbr{s, a}$ and $0$ otherwise. Plugging this feature mapping into~\eqref{eq:dual_lp_embedding}, we recover the original dual \qlp~\eqref{eq:rho_lp}. 

\vspace{-2mm}
\paragraph{Example (Full-rank basis):} Suppose $\Phi\in \RR^{p\times p}$ is a full-rank matrix with $p=\abr{S}\abr{A}$; furthermore, $\phi(s,a) = \Phi((s,a),\cdot)^\top$. Although the constraints in~\eqref{eq:dual_lp_embedding} and~\eqref{eq:rho_lp} are different, their solutions are identical. This can be verified by the Lagrangian in~\appref{appendix:approx_error}.

\vspace{-2mm}
\paragraph{Example (RKHS function mappings):} Suppose $\phi\rbr{s, a}\defeq k\rbr{\rbr{s, a}, \cdot}\in \RR^p$ with $p=\infty$, which forms a reproducing kernel Hilbert space (RKHS) $\Hcal_k$. The LHS and RHS in the constraint of~\eqref{eq:dual_lp_embedding} are the kernel embeddings of $d\rbr{s, a}$ and $\rbr{1 - \gamma}\mu_0\pi\rbr{s, a} + \gamma\cdot\Pcal^\pi_* d\rbr{s, a}$ respectively. The constraint in~\eqref{eq:dual_lp_embedding} can then be understood as as a form of distribution matching by comparing kernel embeddings, rather than element-wise matching as in~\eqref{eq:rho_lp}. If the kernel function $k\rbr{\cdot, \cdot}$ is characteristic, the embeddings of two distributions will match if and only if the distributions are identical almost surely~\citep{SriFukLan11}.

\newcommand{\thmlinearqapprox}{
Suppose the constant function $\one\in \Fcal_\phi\defeq \linspan{\phi}$.  Then,
\begin{equation*}
0\le \tilde\rho_\pi - \rho_\pi \le 2\min_{\beta}\nbr{Q^\pi - \inner{\beta}{\phi}}_\infty,
\end{equation*}
where $Q^\pi$ is the fixed-point solution to the Bellman equation $Q\rbr{s, a} = R\rbr{s, a} + \gamma \Pcal^\pi Q\rbr{s, a}$.}

\begin{theorem}[Approximation error]\label{thm:linear_q_approx}
\thmlinearqapprox
\end{theorem}

Please refer to~\appref{appendix:approx_error} for the proof. The condition $\one\in \Fcal_\phi$ is standard and is trivial to satisfy. Although the approximation error relies on $\nbr{\cdot}_\infty$, a sharper bound that relies on a norm taking the state-action distribution into account can also be obtained~\citep{FarVan03}. We focus on characterizing the uncertainty due to sampling in this paper, so for ease of exposition we will consider a setting where $\phi$ is sufficiently expressive to make the approximation error zero. If desired, the approximation error in~\thmref{thm:linear_q_approx} can be included in the analysis.

Note that, compared to using a characteristic kernel to ensure injectivity for the RKHS embeddings over all distributions (and thus guaranteeing arbitrarily small approximation error),~\thmref{thm:linear_q_approx} only requires that $Q^\pi$ be represented in $\Fcal_\phi$, which is a much weaker condition. In practice, one may also learn the feature mapping $\phi$ for the projection jointly. 

\subsection{Off-policy Confidence Interval Estimation}\label{sec:ciel}

By introducing the function space embedding of the constraints in~\eqref{eq:dual_lp_embedding}, we have transformed the original point-wise constraints in the \qlp to generalized estimating equations. This paves the way to applying the generalized empirical likelihood~(EL)~\citep{Owen01,BroKez12,BerGauHar14,DucGlyNam16} method to estimate a confidence interval on policy value.

Recall that, given a convex, lower-semicontinuous function $f: \RR_+\rightarrow \RR$ satisfying $f\rbr{1} = 0$, the $f$-divergence between densities $p$ and $q$ on $\RR$ is defined as $D_f\rbr{P||Q}\defeq \int Q\rbr{dx}f\rbr{\frac{dP\rbr{x}}{dQ\rbr{x}}}dx$.

Given an $f$-divergence, we propose our main confidence interval estimate based on the following confidence set $C^f_{n, \xi}\subset\RR$: 
\begin{equation}\label{eq:dice_region}
C^f_{n, \xi}\defeq\cbr{\tilde\rho_\pi(w) =\max_{\tau\ge 0} \EE_w\sbr{\tau\cdot r} \bigg| w\in \Kcal_f, \EE_{w}\sbr{\Delta\rbr{x; \tau, \phi}} = 0}, \,\text{ with }\, \Kcal_f\defeq \cbr{
\begin{matrix} 
w\in\Pcal^{n-1}\rbr{\phat_n}, \\
D_f\rbr{w||\phat_n}\le \frac{\confrange}{n}
\end{matrix}},
\end{equation}
where $\Pcal^{n-1}\rbr{\phat_n}$ denotes the $n$-simplex on the support of $\phat_n$, the empirical distribution over $\Dcal$. It is easy to verify that this set $C^f_{n, \xi}\subset\RR$ is convex, since $\tilde\rho_\pi\rbr{w}$ is a convex function over a convex feasible set. Thus, $C^f_{n, \xi}$ is an interval. In fact, $C^f_{n, \xi}$ is the image of the policy value $\tilde\rho_\pi$ on a bounded (in $f$-divergence) perturbation to $w$ in the neighborhood of the empirical distribution $\phat_n$.

Intuitively, the confidence interval $C^f_{n, \xi}$ possesses a close relationship to bootstrap estimators. In vanilla bootstrap, one constructs a set of empirical distributions $\cbr{w^i}_{i=1}^m$ by resampling from the dataset $\Dcal$. Such subsamples are used to form the empirical distribution on $\cbr{\tilde\rho\rbr{w^i}}_{i=1}^m$, which provides population statistics for confidence interval estimation. However, this procedure is computationally very expensive, involving $m$ separate optimizations. By contrast, our proposed estimator $C^f_{n, \xi}$ exploits the asymptotic properties of the statistic $\tilde\rho_\pi\rbr{w}$ to derive a target confidence interval by solving only \emph{two} optimization problems (\secref{sec:comp}), a dramatic savings in computational cost.

Before introducing the algorithm for computing $C^f_{n, \xi}$, we establish the first key result that, by choosing $\confrange = \chi_{(1)}^{2, 1-\alpha}$, $C^f_{n, \confrange}$ is asymptotically a $\rbr{1-\alpha}$-confidence interval on the policy value, where $\chi_{(1)}^{2,1-\alpha}$ is the $\rbr{1-\alpha}$-quantile of the $\chi^2$-distribution with $1$ degree of freedom. 

\newcommand{\thmasymptoticcoverage}{
Under~\asmpsref{asmp:f_div},~\ref{asmp:bounded_ratio}, and~\ref{asmp:phi_regularity}, if $\Dcal$ contains \iid~samples and the optimal solution to the Lagrangian of~\eqref{eq:dual_lp_embedding} is unique, we have
\begin{equation}
\lim_{n\rightarrow\infty} \PP\rbr{\rho_\pi\in C^f_{n, \xi}} = \PP\rbr{\chi_{(1)}^2\le \xi}.
\end{equation}
Therefore, $C^f_{n, \chi_{(1)}^{2, 1-\alpha}}$ is an asymptotic $\rbr{1-\alpha}$-confidence interval of the value of the policy $\pi$.
}
\newcommand{\informalthmasymptoticcoverage}{
Under some mild conditions, if $\Dcal$ contains \iid~samples and the optimal solution to the Lagrangian of~\eqref{eq:dual_lp_embedding} is unique, we have
\begin{equation}
\lim_{n\rightarrow\infty} \PP\rbr{\rho_\pi\in C^f_{n, \xi}} = \PP\rbr{\chi_{(1)}^2\le \xi}.
\end{equation}
Thus, $C^f_{n, \chi_{(1)}^{2, 1-\alpha}}$ is an asymptotic $\rbr{1-\alpha}$-confidence interval of the value of the policy $\pi$.
}

\begin{theorem}[Informal asymptotic coverage]\label{thm:asymptotic_coverage}
\informalthmasymptoticcoverage
\end{theorem}

Please refer to~\appref{appendix:asymptotic} for the precise statement and proof of~\thmref{thm:asymptotic_coverage}. 

\thmref{thm:asymptotic_coverage} generalizes the result in~\citet{DucGlyNam16} to statistics with generalized estimating equations, maintaining the $1$ degree of freedom in the asymptotic $\chi_{(1)}^2$-distribution. One may also apply existing results for EL with generalized estimating equations~\citep[e.g.,][]{LamZhou17}, but these would lead to a limiting distribution of $\chi^2_{(m)}$ with $m \gg 1$ degrees of freedom, resulting in a much looser confidence interval estimate than 
\thmref{thm:asymptotic_coverage}.

Note that \thmref{thm:asymptotic_coverage} can also be specialized to multi-armed contextual bandits to achieve a tighter confidence interval estimate in this special case. In particular, for contextual bandits, the stationary distribution constraint in~\eqref{eq:dual_lp_embedding}, $\EE_{w}\sbr{\Delta\rbr{x;\tau, \phi}} = 0$, is no longer needed, and can be replaced by $\EE_{w}\sbr{\tau - 1} = 0$. Then by the same technique used for MDPs, we can obtain a confidence interval estimate for offline contextual bandits; see details in~\appref{appendix:coin_bandit}. Interestingly, the resulting confidence interval estimate not only has the same asymptotic coverage as previous work \citep{KarLanMin19}, but is also simpler and computationally more efficient.

\subsection{Computing the Confidence Interval}\label{sec:comp}

Now we provide a distributional robust optimization view of the upper and lower bounds of $C_{n, \confrange}^f$.
\begin{theorem}[Upper and lower confidence bounds]\label{thm:upper_lower}
Denote the upper and lower confidence bounds of $C_{n, \confrange}^f$ by $u_n$ and $l_n$, respectively:
\begin{eqnarray}
[l_n, u_n] &=& \sbr{{\color{blue}\min_{w\in\Kcal_f}}\min_{\beta\in \RR^{p}}\max_{\tau\ge 0} \EE_w\sbr{\ell\rbr{x; \tau, \beta} },
\quad
{\color{red}\max_{w\in\Kcal_f}}\max_{\tau\ge 0}\min_{\beta\in \RR^{p}} \EE_w\sbr{\ell\rbr{x; \tau, \beta}}
},\label{eq:upper_lower_outer}\\
&=&\sbr{\min_{\beta\in \RR^{p}}\max_{\tau\ge 0}{\color{blue}\min_{w\in\Kcal_f}} \EE_w\sbr{\ell\rbr{x; \tau, \beta} },
\quad
\max_{\tau\ge 0}\min_{\beta\in \RR^{p}}{\color{red}\max_{w\in\Kcal_f}} \EE_w\sbr{\ell\rbr{x; \tau, \beta}}\label{eq:upper_lower}
},
\end{eqnarray}
where $\ell\rbr{x; \tau, \beta} \defeq \tau\cdot r + \beta^\top \Delta\rbr{x; \tau, \phi}$.
For any $\rbr{\tau, \beta, \lambda, \eta}$ that satisfies the constraints in
\eqref{eq:upper_lower}, the optimal weights for the upper and lower confidence bounds are
\begin{equation}\label{eq:opt_weights}
w_l = f'_*\rbr{\frac{\eta - \ell\rbr{x;\tau, \beta}}{\lambda}}
\quad\text{and}\quad
w_u = f'_*\rbr{\frac{\ell\rbr{x; \tau, \beta} - \eta}{\lambda}}.
\end{equation}
respectively.  Therefore, the confidence bounds can be simplified as:
\begin{eqnarray}\label{eq:simple_upper_lower}
\begin{bmatrix}
l_n \\
u_n
\end{bmatrix}
=
\begin{bmatrix}
\min_{\beta}\max_{\tau\ge 0, \lambda\ge 0, \eta} \EE_{\Dcal}\sbr{-\lambda f_*\rbr{\frac{\eta - \ell\rbr{x; \tau, \beta}}{\lambda}} + \eta - \lambda\frac{\confrange}{n}}\\
\max_{\tau\ge 0} \min_{\beta,\lambda\ge 0, \eta} \EE_{\Dcal}\sbr{\lambda f_*\rbr{\frac{\ell\rbr{x; \tau, \beta} - \eta}{\lambda}} + \eta + \lambda\frac{\confrange}{n} }
\end{bmatrix}
.
\end{eqnarray}
\end{theorem}
The proof of this result relies on Lagrangian duality and the convexity and concavity of the optimization; it may be found in full detail in~\appref{appendix:upper_lower}. 

As we can see in~\thmref{thm:upper_lower}, by exploiting strong duality properties to move $w$ into the inner most optimizations in~\eqref{eq:upper_lower}, the obtained optimization~\eqref{eq:upper_lower} is the distributional robust optimization extenion of the saddle-point problem. The closed-form reweighting scheme is demonstrated in~\eqref{eq:opt_weights}. For particular $f$-divergences, such as the $KL$- and $2$-power divergences, for a fixed $(\beta, \tau)$, the optimal $\eta$ can be easily computed and the weights $w$ recovered in closed-form. For example, by using $KL\rbr{w||\phat_n}$, \eqref{eq:opt_weights} can be used to obtain the updates
\begin{align}
w_l\rbr{x}  = {\exp\rbr{\frac{\eta_l - \ell\rbr{x;\tau, \beta}}{\lambda}}}, 
\quad 
w_u\rbr{x}  = {\exp\rbr{\frac{\ell\rbr{x;\tau, \beta} - \eta_u}{\lambda}}}
,
\end{align}
where $\eta_l$ and $\eta_u$ provide the normalizing constants. (For closed-form updates of $w$ w.r.t.\ other $f$-divergences, please refer to~\appref{appendix:closed_form_reweight}.) Plug the closed-form of optimal weights into~\eqref{eq:upper_lower}, this greatly simplifies the optimization over the data perturbations yielding~\eqref{eq:simple_upper_lower}, and estabilishes the connection to the prioritized experiences replay~\citep{SchQuaAntSil16}, where both reweight the experience data according to their loss, but with different reweighting schemes. 

Note that it is straightforward to check that the estimator for $u_n$ in~\eqref{eq:simple_upper_lower} is nonconvex-concave and the estimator for $l_n$ in~\eqref{eq:simple_upper_lower} is nonconcave-convex. Therefore, one could alternatively apply stochastic gradient descent-ascent~(SGDA) for to solve~\eqref{eq:simple_upper_lower} and benefit from attractive finite-step convergence guarantees~\citep{LinJinJor19}. 

\vspace{-2mm}
\paragraph{Remark (Practical considerations):} As also observed in~\citet{NamDuc16}, SGDA for~\eqref{eq:simple_upper_lower} could potentially suffer from high variance in both the objective and gradients when $\lambda$ approaches $0$. In~\appref{appendix:practical_algorithm}, we exploit several properties of~\eqref{eq:upper_lower}, which leads to a computational efficient algorithm, to overcome the numerical issue. Please refer to~\appref{appendix:practical_algorithm} for the details of~\algref{alg:coindice} and the practical considerations. 

\vspace{-2mm}
\paragraph{Remark (Joint learning for feature embeddings):} The proposed framework also allows for the possibility to learn the features for constraint projection. In particular, consider $\zeta\rbr{\cdot, \cdot} \defeq \beta^\top\phi\rbr{\cdot, \cdot}: S\times A\rightarrow \RR$. Note that we could treat the combination $\beta^\top\phi\rbr{s, a}$ together as the Lagrange multiplier function for the original \qlp with \emph{infinitely} many constraints, hence both $\beta$ and $\phi\rbr{\cdot, \cdot}$ could be updated jointly. Although the conditions for asymptotic coverage no longer hold, the finite-sample correction results of the next section are still applicable. This might offer an interesting way to reduce the approximation error introduced by inappropriate feature embeddings of the constraints, while still maintaining calibrated confidence intervals.

\section{Finite-sample Analysis}\label{sec:analysis}

\thmref{thm:asymptotic_coverage} establishes the asymptotic $\rbr{1-\alpha}$-coverage of the confidence interval estimates produced by~\estname, ignoring higher-order error terms that vanish as sample size $n\to\infty$.  In practice, however, $n$ is always finite, so it is important to quantify these higher-order terms. This section addresses this problem, and presents a finite-sample bound for the estimate of \estname. In the following, we let $\Fcal_\tau$ and $\Fcal_\beta$ be the function classes of $\tau$ and $\beta$ used by \estname.
\newcommand{\informalthmfinitesample}{
Denote by $d_{\Fcal_\tau}$ and $d_{\Fcal_\beta}$ the finite \texttt{VC}-dimension of $\Fcal_\tau$ and $\Fcal_\beta$, respectively. 
Under some mild conditions, when $D_f$ is $\chi^2$-divergence, we have
$$
\PP\rbr{\rho_\pi\in [l_n-\kappa_n, u_n + \kappa_n]} \ge 1 - 12\exp\rbr{c_1 + 2\rbr{d_{\Fcal_\tau} + d_{\Fcal_\beta} -1}\log n - \frac{\confrange}{18}},
$$
where $c_1 = 2c + \log d_{\Fcal_\tau} + \log d_{\Fcal_\beta} + \rbr{d_{\Fcal_\tau}  + d_{\Fcal_\beta} -1}$, $\kappa_n = \frac{11M\confrange}{6n} + 2\frac{C_\ell M}{n}\rbr{1 + 2\sqrt{\frac{\confrange}{9n}}}$, and $\rbr{c, M, C_\ell}$ are univeral constants.  
}
\vspace{-2mm}
\begin{theorem}[Informal finite-sample correction]\label{thm:finite_sample}
\informalthmfinitesample
\vspace{-2mm}
\end{theorem}
The precise statement and detailed proof of~\thmref{thm:finite_sample} can be found in~\appref{appendix:finite_sample}. The proof relies on empirical Bernstein bounds with a careful analysis of the variance term. Compared to the vanilla sample complexity of $\Ocal\rbr{\frac{1}{\sqrt{n}}}$, we achieve a faster rate of $\Ocal\rbr{\frac{1}{n}}$ without any additional assumptions on the noise or curvature conditions. The tight sample complexity in \thmref{thm:finite_sample} implies that one can construct the $\rbr{1 -\alpha}$-finite sample confidence interval by optimizing~\eqref{eq:upper_lower} with $\confrange = 18\rbr{\log \frac{\alpha}{12} - c_1 - 2\rbr{d_{\Fcal_\tau} + d_{\Fcal_\beta} - 1}\log n}$, and composing with $\kappa_n$. However, we observe that this bound can be conservative compared to the asymptotic confidence interval in \thmref{thm:asymptotic_coverage}. Therefore, we will evaluate the asymptotic version of \estname based on \thmref{thm:asymptotic_coverage} in the experiment.

The conservativeness arises from the use of a union bound. However, we conjecture that the rate is optimal up to a constant. We exploit the VC dimension due to its generality. In fact, the bound can be improved by considering a data-dependent measure, \eg, Rademacher complexity, or by some function class dependent measure, \eg, function norm in RKHS, for specific function approximators.

\section{Optimism vs. Pessimism Principle}\label{sec:opt_pes}

\estname provide both upper and lower bounds of the target policy's estimated value, which paves the path for applying the principle of optimism~\citep{LatSze20} or pessimism~\citep{SwaJoa15} in the face of uncertainty for policy optimization in different learning settings. 

\vspace{-2mm}
\paragraph{Optimism in the face of uncertainty.}  Optimism in the face of uncertainty leads to \emph{risk-seeking} algorithms, which can be used to balance the exploration/exploitation trade-off. Conceptually, they always treat the environment as the best plausibly possible. This principle has been successfully applied to stochastic bandit problems, leading to many instantiations of UCB algorithms~\citep{LatSze20}. In each round, an action is selected  according to the upper confidence bound, and the obtained reward will be used to refine the confidence bound iteratively. When applied to MDPs, this principle inspires many optimistic model-based~\citep{BarMen02,AueJakOrt09,Strehl09Pac,SziSze10,DanLatBru17}, value-based~\citep{JinAllBubJor18}, and policy-based algorithms~\citep{CaiYanJinWan19}. Most of these algorithms are not compatible with function approximators. 

We can also implement the optimism principle by optimizing the upper bound in~\estname iteratively, \ie, $\max_{\pi}u_\Dcal\rbr{\pi}$. In $t$-th iteration, we calculate the gradient of $u_{\Dcal}\rbr{\pi^t}$, \ie, $\nabla_{\pi}u_{\Dcal}\rbr{\pi^t}$, based on the existing dataset $\Dcal_t$, then, the policy $\pi_{t}$ will be updated by (natural) policy gradient and samples will be collected through the updated policy $\pi_{t+1}$. Please refer to~\appref{appendix:opt_pes} for the gradient computation and algorithm details. 

\vspace{-2mm}
\paragraph{Pessimism in the face of uncertainty.} In offline reinforcement learning~\citep{LanGabRie12,FujMegPre19,WuTucNac19,NacDaiKosChoetal19}, only a fixed set of data from behavior policies is given, a safe optimization criterion is to maximize the worst-case performance among a set of statistically plausible models~\citep{LarTriDes19,KumFuSohTucetal19,YuThoYuErmetal20}. In contrast to the previous case of online exploration, this is a pessimism principle~\citep{CohHut20,BucGelBel20} or counterfactual risk minimization~\citep{SwaJoa15}, and highly related to robust MDP~\citep{Iyengar05,NilEl05,TamXuMan13,ChoTamManPav15}.

Different from most of the existing methods where the worst-case performance is characterized by model-based perturbation or ensemble, the proposed~\estname provides a lower bound to implement the pessimism principle, \ie, $\max_\pi l_\Dcal\rbr{\pi}$. Conceptually, we apply the (natural) policy gradient w.r.t. $l_\Dcal\rbr{\pi^t}$ to update the policy iteratively. Since we are dealing with policy optimization in the offline setting, the dataset $\Dcal$ keeps unchanged. Please refer to~\appref{appendix:opt_pes} for the algorithm details.

\section{Related Work}\label{sec:related_work}

Off-policy estimation has been extensively studied in the literature, given its practical importance.  Most existing methods are based on the core idea of mportance reweighting to correct for distribution mismatches between the target policy and the off-policy data~\citep{PreSutSin00,BotPetQuiChaetal13,LiMunSze15,xie19optimal}. Unfortunately, when applied naively, importance reweighting can result in an excessively high variance, which is known as the ``curse of horizon''~\citep{LiuLiTanZho18}. To avoid this drawback, there has been rapidly growing interest in estimating the correction ratio of the \emph{stationary} distribution~\citep[e.g.,][]{LiuLiTanZho18,NacChoDaiLi19,UehHuaJia19,liu19understanding,ZhaDaiLiSch20,zhang20gradientdice}. This work is along the same line and thus applicable in long-horizon problems.  Other off-policy approaches are also possible, notably model-based~\citep[e.g.,][]{fonteneau13batch} and doubly robust methods~\citep{jiang16doubly,thomas16data,tang20doubly,UehHuaJia19}.  These techniques can potentially be combined with our algorithm, which we leave for future investigation.

While most OPE works focus on obtaining accurate \emph{point} estimates, several authors provide ways to quantify the amount of uncertainty in the OPE estimates.  In particular, confidence bounds have been developed using the central limit theorem~\citep{BotPetQuiChaetal13}, concentration inequalities~\citep{thomas2015high,kuzborskij2020confident}, and nonparametric methods such as the bootstrap~\citep{thomas2015high_opt,hanna17bootstrapping}. In contrast to these works, the CoinDICE is asymptotically pivotal, meaning that there are no hidden quantities we need to estimate, which is based on correcting for the stationary distribution in the \emph{behavior-agnostic} setting, thus avoiding the curse of horizon and broadening the application of the uncertainty estimator. Recently, \citet{jiang20minimax} provide confidence intervals for OPE, but focus on the intervals determined by the \emph{approximation error} induced by a function approximator, while our confidence intervals quantify \emph{statistical error}.

Empirical likelihood~\citep{Owen01} is a powerful tool with many applications
in statistical inference like econometrics~\citep{chen18monte}, and more recently in distributionally robust optimization~\citep{DucGlyNam16,LamZhou17}. EL-based confidence intervals can be used to guide exploration in multiarmed bandits~\citep{HonTak10,CapGarMaiMunetal13}, and for OPE~\citep{KarLanMin19,kallus19intrinsically}. While the work of \citet{kallus19intrinsically} is also based on EL, it differs from the present work in two important ways. First, their focus is on developing an asymptotically efficient OPE \emph{point} estimate, not confidence intervals.  Second, they solve for timestep-dependent weights, whereas we only need to solve for timestep-\emph{independent} weights from a system of moment matching equations induced by an underlying ergodic Markov chain.

\section{Experiments}\label{sec:experiments}
%
\begin{figure}[h]
  \setlength{\tabcolsep}{0pt}
  \renewcommand{\arraystretch}{0.7}
  \begin{center}
    \begin{tabular}{lcccc}
      & \multicolumn{2}{c}{\bf FrozenLake}
      & \multicolumn{2}{c}{\bf Taxi} \\
    &
    \small $\text{\# trajectories}=50$ &
    \small $\text{\# trajectories}=100$ &
    \small $\text{\# trajectories}=20$ &
    \small $\text{\# trajectories}=50$ \\
    \rotatebox[origin=c]{90}{\tiny interval coverage\hspace{-0.85in}} &
    \includegraphics[width=0.2\textwidth]{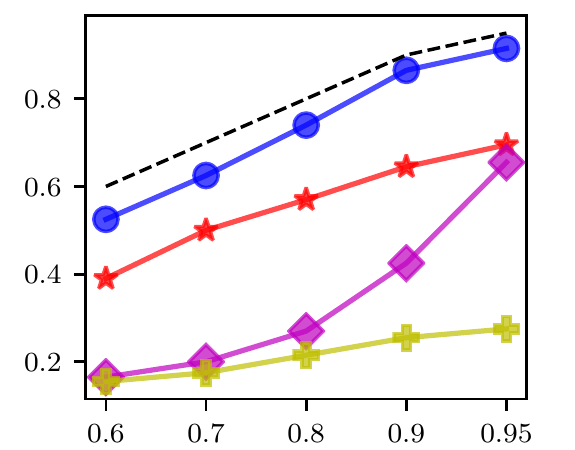} &
    \includegraphics[width=0.2\textwidth]{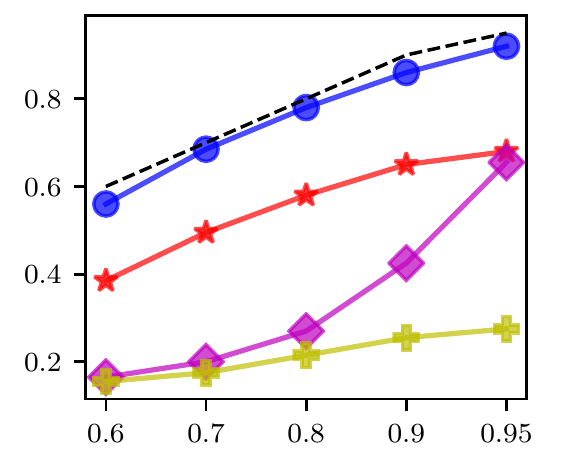} &
    \includegraphics[width=0.2\textwidth]{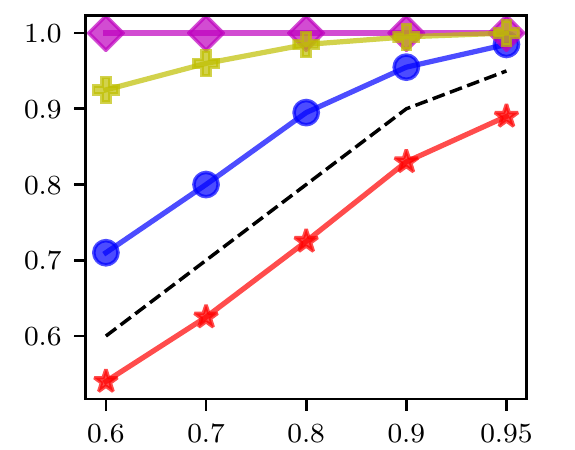} &
    \includegraphics[width=0.2\textwidth]{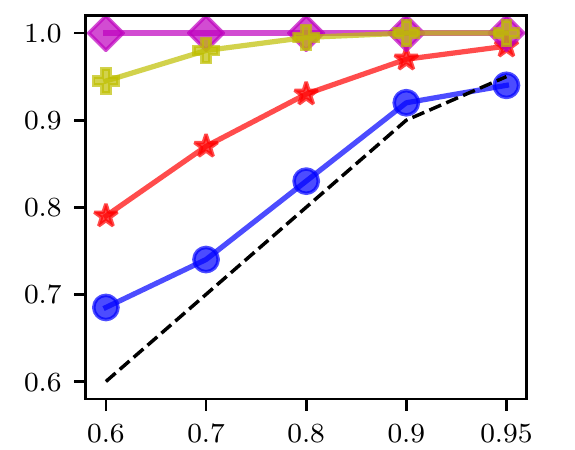} \\
    \rotatebox[origin=c]{90}{\tiny interval log-width\hspace{-0.85in}} &
    \includegraphics[width=0.2\textwidth]{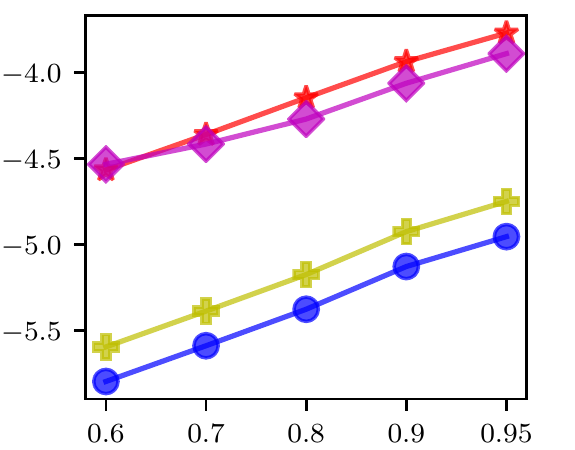} &
    \includegraphics[width=0.2\textwidth]{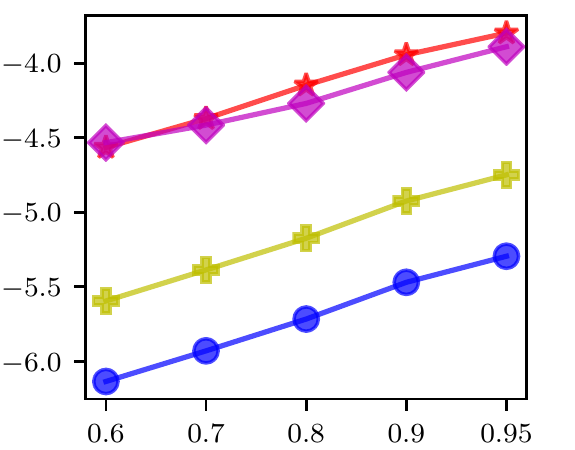} &
    \includegraphics[width=0.2\textwidth]{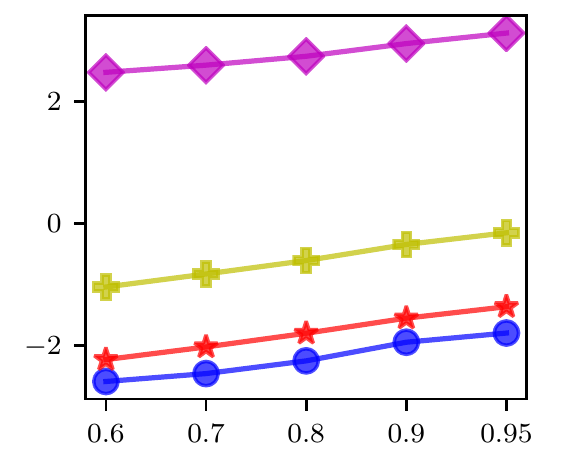} &
    \includegraphics[width=0.2\textwidth]{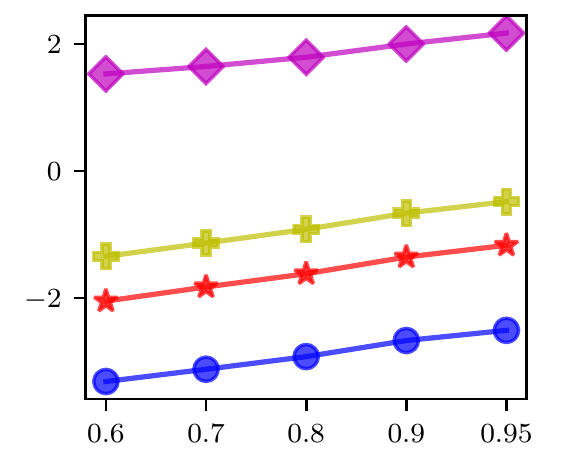} \\
      &  \multicolumn{4}{c}{\tiny Confidence level ($1 \!-\! \alpha$)} \\
      \vspace{-1mm}
    & \multicolumn{4}{c}{\includegraphics[width=0.55\columnwidth]{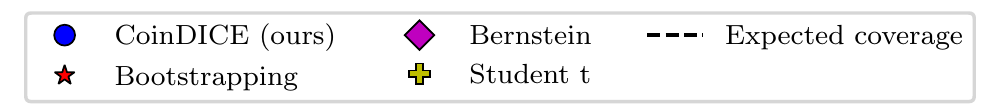}}
    \end{tabular}
  \end{center}
  \vspace{-5mm}
  \caption{Results of \estname and baseline methods on an infinite-horizon version of FrozenLake and Taxi. In FrozenLake, each dataset consists of trajectories of length $100$; in Taxi, each dataset consists of trajectories of length $500$.
  }
  \label{fig:lake-taxi}
  \vspace{-3mm}
\end{figure}

We now evaluate the empirical performance of \estname, comparing it to a number of existing confidence interval estimators for OPE based on concentration inequalities. Specifically, given a dataset of logged trajectories, we first use weighted step-wise importance sampling~\citep{PreSutSin00} to calculate a separate estimate of the target policy value for each trajectory. Then given such a finite sample of estimates, we then use the empirical \emph{Bernstein} inequality~\citep{thomas2015high} to derive high-confidence lower and upper bounds for the true value. Alternatively, one may also use \emph{Student's $t$-test} or Efron's bias corrected and accelerated \emph{bootstrap}~\citep{thomas2015high_opt}.

We begin with a simple bandit setting, devising a two-armed bandit problem with stochastic payoffs. We define the target policy as a near-optimal policy, which chooses the optimal arm with probability $0.95$. We collect off-policy data using a behavior policy which chooses the optimal arm with probability of only $0.55$. Our results are presented in Figure~\ref{fig:bandit}. We plot the empirical coverage and width of the estimated intervals across different confidence levels. More specifically, each data point in Figure~\ref{fig:bandit} is the result of $200$ experiments. In each experiment, we randomly sample a dataset and then compute a confidence interval. The \emph{interval coverage} is then computed as the proportion of intervals out of $200$ that contain the true value of the target policy. The \emph{interval $\log$-width} is the median of the log of the width of the $200$ computed intervals. Figure~\ref{fig:bandit} shows that the intervals produced by \estname achieve an empirical coverage close to the intended coverage. In this simple bandit setting, the coverages of Student's $t$ and bootstrapping are also close to correct, although they suffer more in the low-data regime. Notably, the width of the intervals produced by \estname are especially narrow while maintaining accurate coverage.
\begin{wrapfigure}{R}{0.6\textwidth}
  \setlength{\tabcolsep}{0pt}
  \renewcommand{\arraystretch}{0.7}
  \begin{center}
  \vspace{-6mm}
    \begin{tabular}{lccccc}
    &
    \small $\text{\# samples}=50$ &
    \small $\text{\# samples}=100$ &
    \small $\text{\# samples}=200$ \\
    \rotatebox[origin=c]{90}{\tiny interval coverage\hspace{-0.85in}} &
    \includegraphics[width=0.2\textwidth]{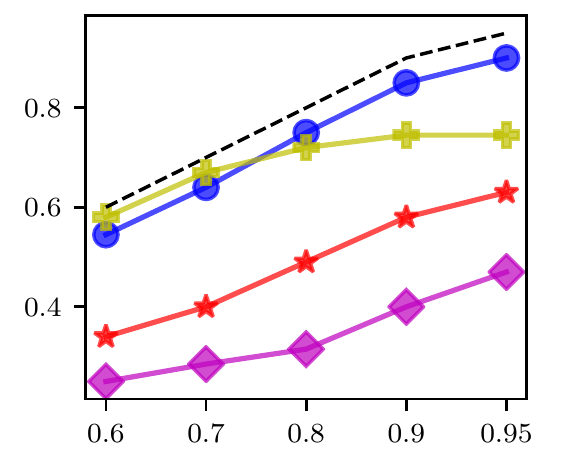} &
    \includegraphics[width=0.2\textwidth]{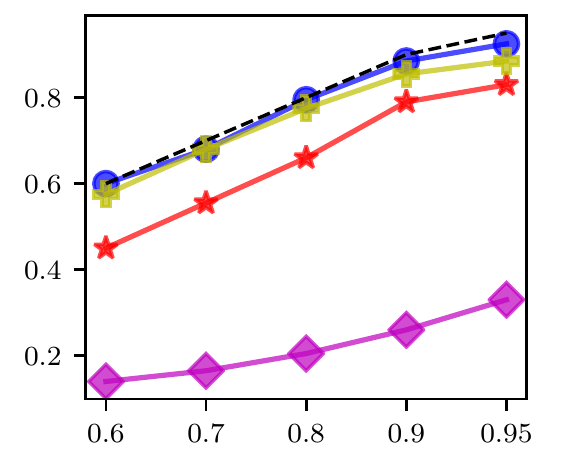} &
    \includegraphics[width=0.2\textwidth]{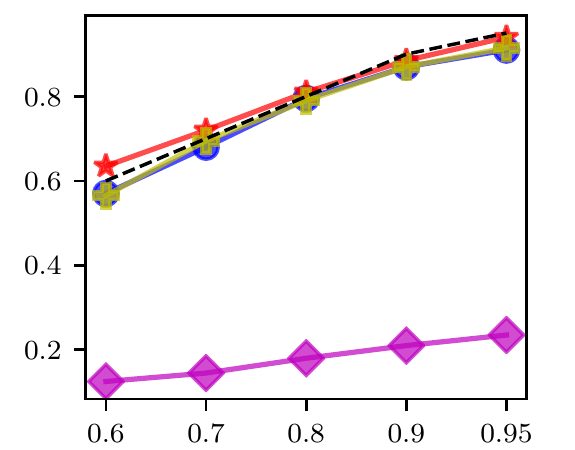} \\
    \rotatebox[origin=c]{90}{\tiny interval log-width\hspace{-0.85in}} &
    \includegraphics[width=0.2\textwidth]{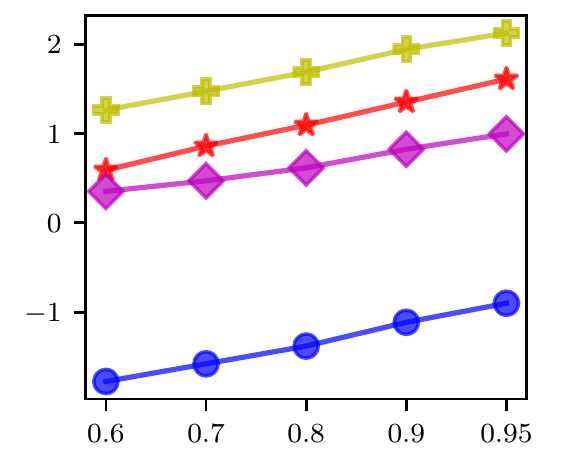} &
    \includegraphics[width=0.2\textwidth]{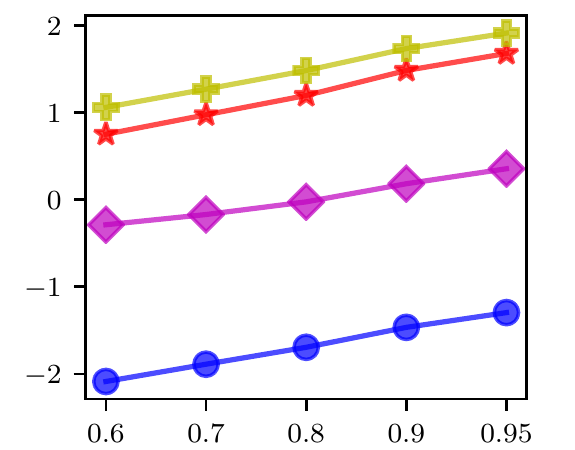} &
    \includegraphics[width=0.2\textwidth]{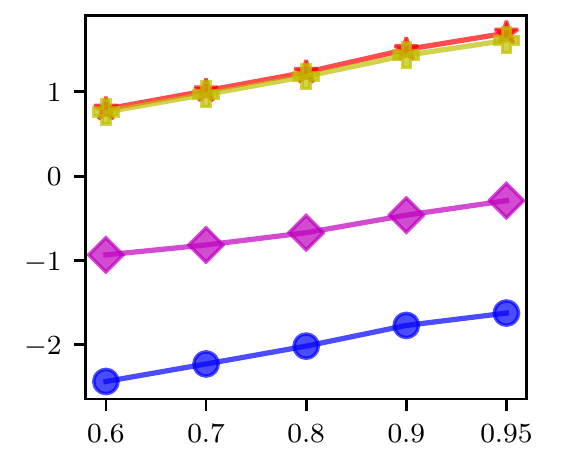} \\
      &  \multicolumn{3}{c}{\tiny Confidence level ($1 \!-\! \alpha$)} \\
      \vspace{-2mm}
    & \multicolumn{3}{c}{\includegraphics[width=0.55\columnwidth]{figs/legend.pdf}}
    \end{tabular}
  \end{center}
  \vspace{-2mm}
  \caption{Results of \estname and baseline methods on a simple two-armed bandit. We plot empirical coverage and median $\log$-width ($y$-axes) of intervals evaluated at a number of desired confidence levels ($x$-axis), as measured over 200 random trials. We find that \estname achieves more accurate coverage and narrower intervals compared to the baseline confidence interval estimation methods.
  }
  \label{fig:bandit}
  \vspace{-2mm}
\end{wrapfigure}

We now turn to more complicated MDP environments. We use FrozenLake~\citep{brockman2016openai}, a highly stochastic gridworld environment, and Taxi~\citep{Dietterich98b}, an environment with a moderate state space of $2\,000$ elements. As in~\citep{LiuLiTanZho18}, we modify these environments to be infinite horizon by randomly resetting the state upon termination. The discount factor is $\gamma=0.99$. The target policy is taken to be a near-optimal one, while the behavior policy is highly suboptimal.  The behavior policy in FrozenLake is the optimal policy with 0.2 white noise, which reduces the policy value dramatically, from ~0.74 to ~0.24. For the behavior policies in Taxi and Reacher, we follow the same experiment setting for constructing the behavior policies to collect data as in~\citep{NacChoDaiLi19,LiuLiTanZho18}.

We follow the same evaluation protocol as in the bandit setting, measuring empirical interval coverage and $\log$-width over $200$ experimental trials for various dataset sizes and confidence levels. Results are shown in Figure~\ref{fig:lake-taxi}. We find a similar conclusion that \estname consistently achieves more accurate coverage and smaller widths than baselines. Notably, the baseline methods' accuracy suffers more significantly compared to the simpler bandit setting described earlier.

\begin{figure}
  \vspace{-3mm}
    \centering
    \setlength{\tabcolsep}{0pt}
    \renewcommand{\arraystretch}{0.7}
    \begin{tabular}{lccc}
      \rotatebox[origin=c]{90}{\tiny interval coverage\hspace{-1.05in}} &
      \includegraphics[width=0.2\textwidth]{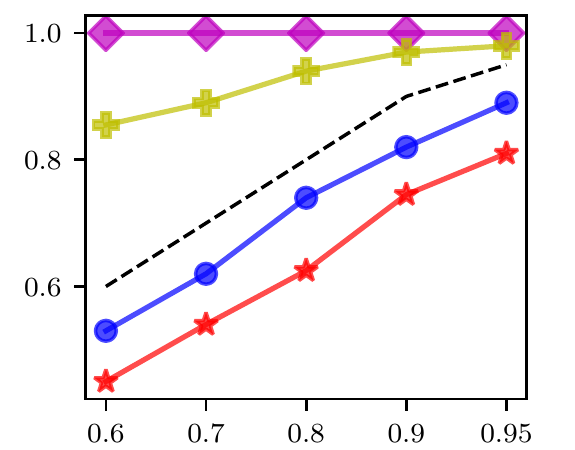} &
      \rotatebox[origin=c]{90}{\tiny interval log-width\hspace{-1.05in}} &
      \includegraphics[width=0.2\textwidth]{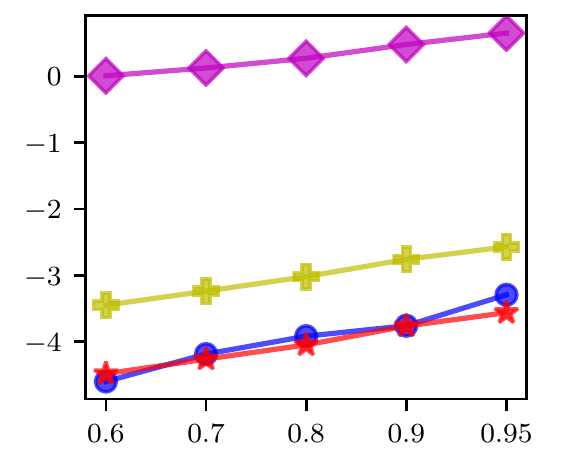} \\
      &  \multicolumn{3}{c}{\tiny Confidence level ($1 - \alpha$)}
    \end{tabular}
      \vspace{-1mm}
      \caption{Results of \estname and baseline methods on Reacher~\citep{brockman2016openai,TodEreTas12}, using $25$ trajectories of length $100$.
    Colors and markers are as defined in the legends of previous figures.}
    \label{fig:reacher}
    \vspace{-2mm}
\end{figure}
Lastly, we evaluate \estname on Reacher~\citep{brockman2016openai,TodEreTas12}, a continuous control environment. In this setting, we use a one-hidden-layer neural network with ReLU activations. Results are shown in Figure~\ref{fig:reacher}. To account for the approximation error of the used neural network, we measure the coverage of \estname with respect to a true value computed as the median of a large ensemble of neural networks trained on the off-policy data. To keep the comparison fair, we measure the coverage of the IS-based baselines with respect to a true value computed as the median of a large number of IS-based point estimates. The results show similar conclusions as before: \estname achieves more accurate coverage than the IS-based methods. Still, we see that \estname coverage suffers in this regime, likely due to optimization difficulties. If the optimum of the Lagrangian is only approximately found, the empirical coverage will inevitably be inexact.

\section{Conclusion}\label{sec:conclusion}

In this paper, we have developed \estname, a novel and efficient confidence interval estimator applicable to the \emph{behavior-agnostic offline} setting. The algorithm builds on a few technical components, including a new feature embedded \qlp, and a generalized empirical likelihood approach to confidence interval estimation. We analyzed the asymptotic coverage of \estname's estimate, and provided an inite-sample bound. On a variety of off-policy benchmarks we empirically compared the new algorithm with several strong baselines and found it to be superior to them.

\section*{Acknowledgements}
We thank Hanjun Dai, Mengjiao Yang and other members of the Google Brain team for helpful discussions. Csaba Szepesv\'ari gratefully acknowledges funding from the Canada CIFAR AI Chairs Program, Amii and NSERC.


\clearpage
\newpage

\appendix
\onecolumn

\begin{appendix}

\begin{center}
{\huge Appendix}
\end{center}

\section{Approximation Error Analysis}\label{appendix:approx_error}

In this section, we provide a complete proof of~\thmref{thm:linear_q_approx}, quantifying the effect of function embedding of constraints in dual \qlp.  The proof is an adaptation from the standard LP for state-value functions to the case of \qlp~\citep{FarVan03}.

We first provide an equivalent reformulation of the primal of the feature embedded LP, 
\begin{lemma}\label{lem:reform_opt}
The solution defined by
$$
\beta^* = \argmin_{\beta\in\RR^p} \cbr{\rbr{1-\gamma}\EE_{\mu_0\pi}\sbr{\beta^\top \phi\rbr{s_0, a_0}} |
\beta^\top\phi\rbr{s, a}\ge \bellman\rbr{\beta^\top\phi}\rbr{s, a}, \,\, \forall \rbr{s, a}\in S\times A
},
$$
with $\rbr{\bellman Q} \rbr{s, a}\defeq R\rbr{s, a} + \gamma\cdot \Pcal^\pi Q\rbr{s, a}$ is also the solution to
\begin{eqnarray}
\min_{\beta\in\RR^p}&& \nbr{Q^\pi - \beta^\top \phi}_{1, \mu_0\pi} \\
  \st &&\beta^\top\phi\rbr{s, a}\ge \bellman \rbr{\beta^\top\phi}\rbr{s, a}, \,\, \forall \rbr{s, a}\in S\times A,\nonumber
\end{eqnarray}
where $\nbr{f}_{1, \mu_0\pi}\defeq \int \abr{f\rbr{s, a}}\mu_0\rbr{s}\pi\rbr{a|s}dsda$. 
\end{lemma}
\begin{proof}
Recall the fact that $\bellman$ is monotonic: given two bounded functions, $\nu_1\ge \nu_2$ implies $\bellman\nu_1\ge \bellman\nu_2$. Therefore, for any feasible $\nu$, we have $\nu\ge\bellman\nu\ge\bellman^2\nu\ge\ldots\ge\bellman^\infty\nu = Q^\pi$, where the convergence to $Q^\pi$ is due to the contraction property of $\bellman$.

Consider a feasible $\beta$, we have
\begin{eqnarray}
\nbr{Q^\pi - \beta^\top \phi}_{1, \mu_0\pi} = \int \rbr{\beta^\top \phi\rbr{s, a} - Q^\pi\rbr{s, a}}\mu_0\rbr{s}\pi\rbr{a|s}dsda,
\end{eqnarray}
which implies minimizing $\EE_{\mu_0\pi}\sbr{\beta^\top \phi}$ is equivalent to minimizing $\nbr{Q^\pi - \beta^\top \phi}_{1, \mu_0\pi}$. 
\end{proof}

\repeatthm{thm:linear_q_approx}{\thmlinearqapprox}

\begin{proof}
We first show the equivalence between function space embedding of dual \qlp and the linear approximation of primal \qlp, which can be easily derived by checking their Lagrangians. Denote 
\begin{eqnarray}\label{eq:lagrangian_embedding}
\lefteqn{l\rbr{d, \beta}\defeq \EE_{d}\sbr{r\rbr{s, a}} + \beta^\top \inner{\phi}{\rbr{1 - \gamma}\mu_0\pi + \gamma\cdot \Pcal_*^\pi d - d}} \\
&=& \rbr{1 - \gamma}\EE_{\mu_0\pi}\sbr{\beta^\top\phi\rbr{s, a}} + \EE_{d}\sbr{r\rbr{s, a} + \gamma\cdot \Pcal^\pi \beta^\top\phi\rbr{s, a} - \beta^\top \phi\rbr{s, a}} \notag \\
&=& \rbr{1 - \gamma}\EE_{\mu_0\pi}\sbr{\Qbeta\rbr{s, a}} + \EE_{d}\sbr{r\rbr{s, a} + \gamma\cdot \Pcal^\pi \Qbeta\rbr{s, a} - \Qbeta\rbr{s, a}}, \notag
\end{eqnarray}
where $\beta\in \RR^{p}$ and $\Qbeta\rbr{s, a}\defeq \beta^\top \phi\rbr{s, a}$.  Since the $l\rbr{d, \beta}$ is convex-concave w.r.t. $\rbr{\beta,d}$, it is also the Lagrangian of primal \qlp with linear parametrization, \ie, 
\begin{eqnarray}\label{eq:primal_lp_embedding}
\min_{\beta\in\RR^p}&& \rbr{1-\gamma}\EE_{\mu_0\pi}\sbr{\beta^\top \phi\rbr{s_0, a_0}}\\
\st&& \beta^\top\phi\rbr{s, a}\ge R\rbr{s, a} + \gamma\cdot \Pcal^\pi \beta^\top\phi\rbr{s, a}, \quad \forall \rbr{s, a}\in S\times A.\nonumber
\end{eqnarray}
By~\lemref{lem:reform_opt}, it is equivalent to solving 
\begin{eqnarray}\label{eq:reform_opt}
\min_{\beta\in\RR^p} && \nbr{Q^\pi - \beta^\top \phi}_{1, \mu_0\pi} \\
  \st &&\beta^\top\phi\rbr{s, a}\ge \bellman \rbr{\beta^\top\phi}\rbr{s, a}, \quad \forall \rbr{s, a}\in S\times A.\nonumber
\end{eqnarray}
We now define
\begin{align*}
  \rbr{d^*, \beta^*} &\defeq \argmax_{d\ge 0}\argmin_\beta l\rbr{d, \beta}, \\
  \tilde\beta &\defeq \argmin_\beta\nbr{Q^\pi - \beta^\top \phi}_\infty, \\
  \epsilon &\defeq \nbr{Q^\pi - \tilde\beta^\top\phi}_\infty,
\end{align*}
and obtain from strong duality that
\[
\EE_{d^*}\sbr{r\rbr{s, a}} = \rbr{1-\gamma}\EE_{\mu_0\pi}\sbr{\rbr{\beta^*}^\top \phi}.
\]

Recall the fact $\bellman$ is a $\gamma$-contraction operator with the norm $\nbr{\cdot}_\infty$, and we have
\[
\nbr{\bellman\rbr{{\tilde\beta}^\top \phi} - Q^\pi}_\infty\le \gamma\nbr{{\tilde\beta}^\top \phi - Q^\pi}_\infty,
\]
which implies
\[
\bellman\rbr{{\tilde\beta}^\top\phi}\le Q^\pi + \gamma\epsilon \one.
\]
Now consider a new solution $\rbr{{\tilde\beta}^\top\phi - c\one}$, which must be in $\linspan{\phi}$ as $\one\in\linspan{\phi}$.  Then,
\begin{eqnarray*}
\bellman\rbr{{\tilde\beta}^\top\phi - c\one} &=&\bellman\rbr{{\tilde\beta}^\top \phi} - \gamma c\one \\
&\le& Q^\pi + \gamma \epsilon\one - \gamma c\one\\
&\le& \tilde\beta^\top\phi + \rbr{1 + \gamma}\epsilon\one - \gamma c\one\\
&=& \tilde\beta^\top\phi - c\one + \rbr{\rbr{1 - \gamma}c + \rbr{1 + \gamma}\epsilon}\one.  
\end{eqnarray*}
Choose $c=-\rbr{1+\gamma}\epsilon/\rbr{1-\gamma}$, and the above implies $\bellman\rbr{{\tilde\beta}^\top\phi - c\one}  \le \tilde\beta^\top\phi - c\one$.  Therefore, there exists some $\bar{\beta}$ such that
\[
\bar\beta^\top \phi = \tilde\beta^\top \phi + \frac{1+\gamma}{1 -\gamma}\epsilon\one.
\]
Then, we can bound the approximation error
\begin{eqnarray*}
\EE_{d^*}\sbr{r\rbr{s, a}} - \rho_\pi&=&\EE_{d^*}\sbr{r\rbr{s, a}} -\rbr{1 - \gamma}\EE_{\mu_0\pi}\sbr{Q^\pi}\\
&=&\rbr{1 - \gamma}\EE_{\mu_0\pi}\sbr{\rbr{\beta^*}^\top\phi} - \rbr{1 - \gamma}\EE_{\mu_0\pi}\sbr{Q^\pi}\ge 0,
\end{eqnarray*}
where the last inequality comes from the fact $\rbr{1 - \gamma}\EE_{\mu_0\pi}\sbr{\rbr{\beta^*}^\top\phi}$ is the optimal value of a restricted feasible set within linearly representable $\Qbeta$.

On the other hand, we bound
\begin{eqnarray*}
\rbr{1 - \gamma}\EE_{\mu_0\pi}\sbr{\rbr{\beta^*}^\top\phi} - \rbr{1 - \gamma}\EE_{\mu_0\pi}\sbr{Q^\pi}&=&\rbr{1 -\gamma}\nbr{\rbr{\beta^*}^\top\phi - \Qbeta}_{1, \mu_0\pi}\\
&\le& \rbr{1 -\gamma}\nbr{{\bar\beta}^\top\phi - \Qbeta}_{1, \mu_0\pi}\\
&\le& \rbr{1 -\gamma}{\nbr{{\bar\beta}^\top\phi - \Qbeta}_{\infty}}\\
&\le& \rbr{1 -\gamma}\rbr{\nbr{{\bar\beta}^\top\phi - \tilde\beta^\top\phi}_{\infty} + \nbr{Q^\pi - \tilde\beta^\top\phi}_{\infty}}\\
&\le& \rbr{1-\gamma}\rbr{1 + \frac{1 + \gamma}{1 - \gamma}}\epsilon = 2\epsilon.
\end{eqnarray*}
where the third inequality comes from the optimality of~\eqref{eq:reform_opt}. 
\end{proof}

\paragraph{Justification of full-rank basis embedding.} The effect of full-rank basis embedding in the example in~\secref{sec:lp_embedding} can be justified straightforwardly. We consider the Lagrangian~\eqref{eq:lagrangian_embedding}. If the $\phi\in \RR^{\abr{S}\abr{A}\times\abr{S}\abr{A}}$ is full-rank, $\phi^{-1}$ exists. For arbitrary $Q\in \RR^{\abr{S}\abr{A}\times 1}$, there exists $\beta = \rbr{Q\phi^{-1}}^\top$, which means there is an one-to-one correspondence between $Q$ and $\beta$ in Lagrangian. Therefore, in finite state and action MDP, the Lagrangian is not affected by full-rank basis embedding, and therefore, the solution of full-rank basis embedding will be the same as the original LP.

\section{\estname for Undiscounted and finite-horizon MDPs }\label{appendix:undiscounted}

In the main text, we consider the~\estname for infinite-horizon MDP with discounted factor $\gamma \in (0, 1)$. The proposed~\estname can be easily generalized for undiscounted MDPs with $\gamma =1$ and finite-horizon MDPs.

\paragraph{Undiscounted MDP.}
Particularly, we have the dual form of the $Q$-LP as
\begin{equation}\label{eq:dual_lp_undiscounted}
\tilde\rho_\pi \defeq \cbr{\max_{d:S\times A\rightarrow \RR_+}\,\,\EE_d[r\rbr{s, a}]\bigg| \, 
\begin{matrix}
\int d\rbr{s, a}dsda = 1\\
d\rbr{s, a} = \Pcal_*^\pi d\rbr{s, a}, \forall \rbr{s, a}\in S\times A
\end{matrix}
}.
\end{equation}
Comparing with the~\eqref{eq:rho_lp}, we have an extra normalization constraint to avoid the scaling issues. Specifically, if $d(s, a)$ is feasible, without the normalization constraint, $c\cdot d\rbr{s, a}$ will also be feasible for $\forall c>0$. Therefore, the optimization could be unbounded. 

By change-of-variable $\tau\rbr{s, a} = \frac{d^\pi\rbr{s, a}}{d^\Dcal\rbr{s, a}}$ and feature embeddings of the stationary constraint in~\eqref{eq:dual_lp_undiscounted}, we obtain 
\begin{equation}\label{eq:dual_lp_undiscounted_embedding}
\tilde\rho_\pi \defeq \cbr{\max_{\tau:S\times A\rightarrow \RR_+}\,\,\EE_{d^\Dcal}[\tau\cdot r\rbr{s, a}]\bigg| \, 
\begin{matrix}
\EE_{d^\Dcal}\sbr{\tau\rbr{s, a}} = 1\\
\EE_{d^\Dcal}\sbr{\phi\rbr{s', a'}\rbr{\tau\rbr{s', a'} - \tau\rbr{s, a}}}= 0
\end{matrix}
}.
\end{equation}
Then, the~\estname confidence interval is achieved by applying the generalized empirical likelihood to~\eqref{eq:dual_lp_undiscounted_embedding}, \ie, 
\begin{equation}\label{eq:dice_region_undiscounted}
C^f_{n, \xi}\defeq\cbr{\tilde\rho_\pi(w) =\max_{\tau\ge 0} \EE_w\sbr{\tau\cdot r} \bigg| 
\begin{matrix}
w\in \Kcal_f, \EE_w\sbr{\tau - 1} = 0\\
\EE_{w}\sbr{\bar\Delta\rbr{x; \tau, \phi}} = 0
\end{matrix}
}, \,\text{ with }\, \Kcal_f\defeq \cbr{
\begin{matrix} 
w\in\Pcal^{n-1}\rbr{\phat_n}, \\
D_f\rbr{w||\phat_n}\le \frac{\confrange}{n}
\end{matrix}},
\end{equation}
where $\bar\Delta\rbr{x; \tau, \phi}\defeq \phi\rbr{s', a'}\rbr{\tau\rbr{s', a'} - \tau\rbr{s, a}}$.  

As we will discussed in~\secref{sec:comp} for discounted MDPs, similar argument can be applied for~\eqref{eq:dice_region_undiscounted}, and thus, one can have the confidence interval in undiscouted MDPs as
$$
C^f_{n, \xi} = \sbr{l_n, u_n}
$$
with 
\begin{eqnarray}
[l_n, u_n] = \sbr{\min_{\beta\in \RR^{p}, \nu}\max_{\tau\ge 0}{\color{blue}\min_{w\in\Kcal_f}} \EE_w\sbr{\ell\rbr{x; \tau, \beta, \nu} },
\quad
\max_{\tau\ge 0}\min_{\beta\in \RR^{p}, \nu}{\color{red}\max_{w\in\Kcal_f}} \EE_w\sbr{\ell\rbr{x; \tau, \beta, \nu}}},
\end{eqnarray}
where $\ell\rbr{x; \tau, \beta, \nu}\defeq \tau\cdot r + \beta^\top \Delta\rbr{x; \tau, \phi} + \nu - \nu\cdot{\tau}$.

\paragraph{Remark (Normalization constraint):} Although in the discounted MDPs, there is no scaling issue, and thus the normalizaiton constraint is redudant, we still prefer to add the constraint in practice. It does not only bring the benefits in optimization, but also enforce the normalization explicitly and reduce the feasible set, leading to better statistical property. 

\paragraph{Finite-horizon MDP.} While we mainly focus on infinite-horizon MDPs with a discounted factor, the dual method can be adapted to finite-horizon settings straightforwardly. For example, we have the finite-horizon $d$-LP as 
\begin{eqnarray}
\max_{d_h\rbr{s, a}: S\times A\rightarrow \RR_+} &&\sum_{h=1}^H\EE_{d_h}\sbr{r_h\rbr{s, a}}\\
\st&&\,\, d_0\rbr{s, a} = \mu_0\rbr{s}\pi\rbr{a|s}, \\
&&d_{h+1}\rbr{s, a} = \Pcal_*^\pi d_h\rbr{s, a},  
\,\, \forall h\in \cbr{1, \ldots, H}. 
\end{eqnarray}
Upon this finite-horizon formulation, we can derive the finite-step CoinDICE following the same technique, \ie. 
\begin{equation*}
[l_n, u_n] = \sbr{{\color{blue}\min_{w\in\Kcal_f}}\min_{\beta_{h=1}^H\in \RR^{p}}\max_{\tau_{h=1}^H\ge 0} \EE_w\sbr{\ell_H\rbr{x; \tau_{h=1}^H, \beta_{h=1}^H} },\,
{\color{red}\max_{w\in\Kcal_f}}\max_{\tau_{h=1}^H\ge 0}\min_{\beta_{h=1}^H\in \RR^{p}} \EE_w\sbr{\ell_H\rbr{x; \tau_{h=1}^H, \beta_{h=1}^H}}
},
\end{equation*}
where $x\defeq\cbr{\rbr{s, a, r, s', a', h}_{h=1}^H}$, $\ell_H\rbr{x; \tau_{h=1}^H, \beta_{h=1}^H}\defeq \sum_{h=1}^H \tau_h r_h + \sum_{h=1}^H \beta_h^\top \Delta_h\rbr{x; \tau_h, \phi}$, and $\Delta_h\rbr{x; \tau_h, \phi}\defeq \tau_h\rbr{s, a}\phi\rbr{s', a'} - \tau_{h+1}\rbr{s', a'}\phi\rbr{s', a'}$.

\section{CoinBandit}\label{appendix:coin_bandit}

MDPs are strictly more general than multi-armed and contextual bandits. Therefore, our estimator can also be specialized accordingly for confidence interval estimation in bandit problems with slight modifications. Without loss of generality, we consider the contextual bandit setting, while the multi-armed bandits can be further reduced from contextual bandit.

Specifically, in the behavior-agnostic contextual bandit setting, the stationary distribution constraint in~\eqref{eq:dual_lp_embedding} is no long applicable in bandit setting. We rewrite the policy value as
\begin{eqnarray}\label{eq:dual_lp_bandit}
\tilde\rho_\pi&\defeq& \EE_{s\sim \mu^\Dcal, a\sim \pi\rbr{a|s}}\sbr{r\rbr{s, a}} \nonumber \\
&=&  
\cbr{\max_{\tau:S\times A\rightarrow\RR_+} \,\,\EE_{d^\Dcal }\sbr{\tau\cdot r\rbr{s, a}} \Big| \,\,{d^\Dcal}{\cdot\tau} = {\mu^\Dcal\pi}, \EE_{d^\Dcal}\sbr{\tau} = 1 },
\end{eqnarray}
where we reload the $\mu^\Dcal$ as the contextual distribution, which is unchanged for all policies, $d^\Dcal\rbr{s, a} =\mu^\Dcal\rbr{s}\pi_b\rbr{a|s}$, $\tau\rbr{s, a}\defeq \frac{\mu^\Dcal\rbr{s}\pi\rbr{a|s}}{\mu^\Dcal\rbr{s}\pi_b\rbr{a|s}}$, and $\phi\rbr{s, a}$ denotes the feature mappings. We keep the normalization constraint to ensure the validation of density ratio empirically. 

We apply the same technique to~\eqref{eq:dual_lp_bandit}, leading to the~\emph{CoinBandit} confidence interval estimator
\begin{equation}\label{eq:bandit_region}
C^f_{n, \xi}\defeq\cbr{\tilde\rho_\pi(w) =\max_{\tau\ge 0} \EE_w\sbr{\tau\cdot r} \bigg| 
\begin{matrix}
w\in \Kcal_f, \EE_w\sbr{\tau -1} = 0\\
\EE_{w}\sbr{\blacktriangle\rbr{x; \tau, \phi}} = 0
\end{matrix}
}, \,\text{ with }\, \Kcal_f\defeq \cbr{
\begin{matrix} 
w\in\Pcal^{n-1}\rbr{\phat_n}, \\
D_f\rbr{w||\phat_n}\le \frac{\confrange}{n}
\end{matrix}},
\end{equation}
where the $x \defeq \rbr{s, a, s', a'}$ is constructed by $s\sim\mu^\Dcal\rbr{s}, a\sim \pi\rbr{a|s}$ and $\rbr{s', a'}\sim d^\Dcal$, and $\blacktriangle\rbr{x; \tau, \phi}\defeq {\phi\rbr{s, a}} - \phi\rbr{s', a'}\cdot\tau\rbr{s', a'}$. 

Similarly, the interval estimator in CoinBandit~\eqref{eq:bandit_region} can be calculated by solving a minimax optimization.

\paragraph{Remark (Behavior-known contextual bandit):} When the behavior policy $\pi_b\rbr{a|s}$ is known, the solution to~\eqref{eq:dual_lp_bandit} can be computed in closed-form as $\tau\rbr{s, a} = \frac{\pi\rbr{a|s}}{\pi_b\rbr{a|s}}$. Then, the CoinBandit reduces to 
\begin{equation}\label{eq:bandit_region_known}
C^f_{n, \xi}\defeq\cbr{\tilde\rho_\pi(w) =\EE_w\sbr{\tau\cdot r} \bigg| 
\begin{matrix}
w\in \Kcal_f, \\
\EE_w\sbr{\tau - 1}=0
\end{matrix}
}, \,\text{ with }\, \Kcal_f\defeq \cbr{
\begin{matrix} 
w\in\Pcal^{n-1}\rbr{\phat_n}, \\
D_f\rbr{w||\phat_n}\le \frac{\confrange}{n}
\end{matrix}}.
\end{equation}

\paragraph{Remark (Multi-armed bandit):} Furthermore, these estimators~\eqref{eq:bandit_region} and~\eqref{eq:bandit_region_known} can be further reduced for multi-armed bandit. Specifically, we set all $s$ equivalent, then, the $s$ becomes the dummy variable. The CoinBandit estimators~\eqref{eq:bandit_region} and~\eqref{eq:bandit_region_known} reduces for the off-policy evaluation in multi-armed bandit. If the action number is finite, we can use tabular representation for $\tau\rbr{a}$, eliminating the approximation error.

\paragraph{Remark (Comparison to~\citet{KarLanMin19}):} \citet{KarLanMin19} considers the off-policy contextual bandit confidence interval estimation. 
Although both CoinBandit and the estimator in \citet{KarLanMin19} share the same asymptotic coverage, there are significant differences:
\begin{itemize}[leftmargin=*]
  \item The estimator in \citet{KarLanMin19} is derived from empirical likelihood with reverse $KL$-divergence, while our CoinBandit is based on generalized empirical likelihood with arbitrary $f$-divergence.
  \item More importantly, compared to our CoinBandit, which is applicable for both \emph{behavior-agnostic} and \emph{behavior-known} off-policy setting, the estimator in~\citet{KarLanMin19} is only valid for behavior-known setting. 
  \item Computationally, the estimator in~\citet{KarLanMin19} requires an extra statistics, \ie, 
  $$
  \cbr{\max_{w} \sum_{i=1}^n \log \rbr{nw_i}\big| \EE_{w}\sbr{\tau-1} =0,\,\, w\in \Kcal_{-2\log\rbr{\cdot}}},
  $$
  while such quantity is not required in CoinBandit, and thus saving the computational cost. 
  \item Statistically, we provide finite sample complexity for CoinBandit in~\thmref{thm:finite_sample}, while such sample complexity is not clear for~\citet{KarLanMin19}.
\end{itemize}

\section{Stochastic Confidence Interval Estimation}\label{appendix:conf_int}

We analyze the properties of the optimization for the upper and lower bounds and derive the practical algorithm in this section. 

\subsection{Upper and Lower Confidence Bounds}\label{appendix:upper_lower}

We first establish the distribution robust optimization representation of the confidence region:
\begin{lemma}\label{lem:dice_region_II}
Let $\rholag\rbr{w} = \max_{\tau\ge 0}\min_{\beta\in \RR^{p}} \EE_w\sbr{\tau\cdot r + \beta^\top \Delta\rbr{x; \tau, \phi}}$.
The confidence region $C_{n, \confrange}^f$ can be represented equivalently as
\begin{equation}
C_{n,\confrange}^f = \cbr{\rholag\rbr{w}\big|w\in\Kcal_f}.
\end{equation}
\end{lemma}
\begin{proof}
For any $w\in \Kcal_f$, we rewrite the optimization~\eqref{eq:dice_region} by its Lagrangian, which will be an estimate of the policy value,
\begin{equation}
\rholag\rbr{w} = \max_{\tau\ge 0}\min_{\beta\in \RR^{p}} \EE_w\sbr{\tau\cdot r + \beta^\top \Delta\rbr{x; \tau, \phi}}.
\end{equation}
\end{proof}

Based on~\lemref{lem:dice_region_II}, we can formulate the upper and lower bounds:

{\bf \thmref{thm:upper_lower}}
\textit{
Denote the upper and lower confidence bounds of $C_{n, \confrange}^f$ by $u_n$ and $l_n$, respectively:
\begin{eqnarray*}
[l_n, u_n] &=& \sbr{{\color{blue}\min_{w\in\Kcal_f}}\min_{\beta\in \RR^{p}}\max_{\tau\ge 0} \EE_w\sbr{\ell\rbr{x; \tau, \beta} },
\quad
{\color{red}\max_{w\in\Kcal_f}}\max_{\tau\ge 0}\min_{\beta\in \RR^{p}} \EE_w\sbr{\ell\rbr{x; \tau, \beta}}
},\nonumber\\
&=&\sbr{\min_{\beta\in \RR^{p}}\max_{\tau\ge 0}{\color{blue}\min_{w\in\Kcal_f}} \EE_w\sbr{\ell\rbr{x; \tau, \beta} },
\quad
\max_{\tau\ge 0}\min_{\beta\in \RR^{p}}{\color{red}\max_{w\in\Kcal_f}} \EE_w\sbr{\ell\rbr{x; \tau, \beta}}
},
\end{eqnarray*}
where $\ell\rbr{x; \tau, \beta} \defeq \tau\cdot r + \beta^\top \Delta\rbr{x; \tau, \phi}$. For any $\rbr{\tau, \beta, \lambda, \eta}$ that satisfies the constraints in \eqref{eq:upper_lower}, the optimal weights for upper and lower confidence bounds are
\begin{equation*}
w_l = f'_*\rbr{\frac{\eta - \ell\rbr{x;\tau, \beta}}{\lambda}}
\quad\text{and}\quad
w_u = f'_*\rbr{\frac{\ell\rbr{x; \tau, \beta} - \eta}{\lambda}},
\end{equation*}
respectively.  Therefore, the confidence bounds can be simplified as:
\begin{eqnarray*}
\begin{bmatrix}
l_n \\
u_n
\end{bmatrix}
=
\begin{bmatrix}
\min_{\beta}\max_{\tau\ge 0, \lambda\ge 0, \eta} \EE_{\Dcal}\sbr{-\lambda f_*\rbr{\frac{\eta - \ell\rbr{x; \tau, \beta}}{\lambda}} + \eta - \lambda\frac{\confrange}{n}}\\
\max_{\tau\ge 0} \min_{\beta,\lambda\ge 0, \eta} \EE_{\Dcal}\sbr{\lambda f_*\rbr{\frac{\ell\rbr{x; \tau, \beta} - \eta}{\lambda}} + \eta + \lambda\frac{\confrange}{n} }
\end{bmatrix}
.
\end{eqnarray*}
}

\begin{proof} 
We first calculate the upper bound $u_n$ using~\lemref{lem:dice_region_II}:
\begin{eqnarray}
u_n = \max_{w\in \Kcal_f}\rho_\pi\rbr{w} &=& \max_{w\in\Kcal_f}\max_{\tau\ge 0}\min_{\beta\in \RR^{p}} \EE_w\sbr{\tau\cdot r + \beta^\top \Delta\rbr{x; \tau, \phi}}  \nonumber \\
&=& \max_{\tau\ge 0}\max_{w\in\Kcal_f}\min_{\beta\in \RR^{p}} \EE_w\sbr{\tau\cdot r + \beta^\top \Delta\rbr{x; \tau, \phi}}\label{eq:upper_step_II}\\
&=& \max_{\tau\ge 0}\min_{\beta\in \RR^{p}}\max_{w\in\Kcal_f} \EE_w\sbr{\tau\cdot r + \beta^\top \Delta\rbr{x; \tau, \phi}},\label{eq:upper_step_III}
\end{eqnarray}
where the switch between $\max_{w\in\Kcal_f}$ and $\max_{\tau\ge 0}$ in~\eqref{eq:upper_step_II} is immediate, \eqref{eq:upper_step_III} is due to the fact that the objective is concave w.r.t. $\beta$ and convex w.r.t. $w$ and $\tau$, separately. 

We apply Lagrangian to the inner constrained optimization over $w$, leading to
\begin{eqnarray}
u_n &=& \max_{\tau}\min_{\beta, \lambda\ge 0, \eta}\max_{w\ge 0 }\,\, \EE_{w}\sbr{\tau\cdot r + \beta^\top \Delta\rbr{x; \tau, \phi}} - \lambda\rbr{D_f\rbr{w||\phat_n} - \frac{\confrange}{n}} + \eta\rbr{1 - w^\top\one}\nonumber \\
&=&\max_{\tau\ge 0} \min_{\beta,\lambda\ge 0, \eta} \EE_{\Dcal}\sbr{\lambda f_*\rbr{\frac{\tau\cdot r + \beta^\top \Delta\rbr{x; \tau, \phi} - \eta}{\lambda}} + \eta + \frac{\lambda\confrange}{n} },\label{eq:upper_lagrangian}
\end{eqnarray}
where the last equation comes from the conjugate of $f$, and for any given $\rbr{\tau, \beta, \lambda, \eta}$, the optimal $w^*$ will be 
$$
w_{u}^* = f'_*\rbr{\frac{\tau\cdot r + \beta^\top \Delta\rbr{x; \tau, \phi} - \eta}{\lambda}}.
$$

The lower bound $l_n$ may be obtained in a similar fashion:
\begin{eqnarray*} 
l_n = \min_{w\in\Kcal_f} \rho\rbr{w; \pi} &=& \min_{w\in\Kcal_f}\max_{\tau\ge 0}\min_{\beta\in \RR^{p}} \EE_w\sbr{\tau\cdot r + \beta^\top \Delta\rbr{x; \tau, \phi}}\\
&=& \min_{w\in\Kcal_f}\min_{\beta\in \RR^{p}}\max_{\tau\ge 0} \EE_w\sbr{\tau\cdot r + \beta^\top \Delta\rbr{x; \tau, \phi}} \\ 
&=& \min_{\beta\in \RR^{p}}\min_{w\in\Kcal_f}\max_{\tau\ge 0} \EE_w\sbr{\tau\cdot r + \beta^\top \Delta\rbr{x; \tau, \phi}} \\ 
&=& \min_{\beta\in \RR^{p}}\max_{\tau\ge 0}\min_{w\in\Kcal_f} \EE_w\sbr{\tau\cdot r + \beta^\top \Delta\rbr{x; \tau, \phi}}. 
\end{eqnarray*}

Again, we consider the Lagrangian 
\begin{eqnarray*}
l_n &=& \min_{\beta\in \RR^{p}}\max_{\tau\ge 0, \lambda\ge 0, \eta}\min_{w\ge 0 }\,\, \EE_{w}\sbr{\tau\cdot r + \beta^\top \Delta\rbr{x; \tau, \phi}} + \lambda\rbr{D_f\rbr{w||\phat_n} - \frac{\confrange}{n}} + \eta\rbr{1 - w^\top\one}\nonumber\\
&=& \min_{\beta}\max_{\tau\ge 0, \lambda\ge 0, \eta} \EE_{\Dcal}\sbr{-\lambda f_*\rbr{\frac{\eta - \rbr{\tau\cdot r + \beta^\top \Delta\rbr{x; \tau, \phi}}}{\lambda}} + \eta - \frac{\lambda\confrange}{n}},
\end{eqnarray*}
and the optimal weight is
$$
w_l^* = f'_*\rbr{\frac{\eta - \rbr{\tau\cdot r + \beta^\top \Delta\rbr{x; \tau, \phi}}}{\lambda}}. 
$$
\end{proof}

\subsection{Closed-form Solution for Reweighting}\label{appendix:closed_form_reweight}
We consider a few examples of $f$-divergences in \thmref{thm:upper_lower}, and show how the weights can be efficiently computed, for a given $\tau$ and $\beta$.
\begin{itemize}[topsep=1pt, itemsep=0pt,leftmargin=*]
  \item {\bf $KL$-divergence.} To satisfy the conditions in Assumption \ref{asmp:f_div}, we select $f\rbr{x} = 2x\log x $. Recall the property that for any convex function $f$ and any $\alpha>0$, the conjugate function of $g(x)=\alpha f(x)$ is equal to $g_*(y)=\alpha f_*(y/\alpha)$. Let $f$ be the standard $f$-divergence function of $KL$-divergence $\operatorname{KL}\rbr{w||\phat_n}$, i.e., $f\rbr{x} = 2x\log x $. With $g^\prime_*(y)=f^\prime_*(y/\alpha)$, equation \eqref{eq:opt_weights} implies that the following upper and lower bounds:
  \begin{align*}
  w_l\rbr{x}  &= {\exp\rbr{\frac{\eta_l - \ell\rbr{x;\tau, \beta}}{2\lambda}}}, \quad \eta_l = -\log\sum_{i=1}^n\exp\rbr{\frac{- \ell\rbr{x;\tau, \beta}}{2\lambda}}\\
  w_u\rbr{x}  &= {\exp\rbr{\frac{\ell\rbr{x;\tau, \beta} - \eta_u}{2\lambda}}}, \quad \eta_u = \log\sum_{i=1}^n\exp\rbr{\frac{\ell\rbr{x;\tau, \beta}}{2\lambda}}.
  \end{align*}
  This can also be verified by plugging the $f\rbr{x} = 2x\log x$ into~\eqref{eq:opt_weights} and considering $w^\top \one =1$. 

  \item {\bf Reverse KL-divergence.} With the f-divergence function $f\rbr{x} = -\log x$ for the reverse-KL divergence, one has the following upper and lower bounds:
  \begin{align*}
    &w_l\rbr{x} = \lambda\delta\rbr{{ \ell\rbr{x; \tau, \beta} > \eta_l  }}\rbr{\ell\rbr{x; \tau, \beta} - \eta_l}^{-1}, \\
    &\sum_{i=1}^n \delta\rbr{{\ell\rbr{x; \tau, \beta} > \eta_l }}\rbr{\ell\rbr{x; \tau, \beta} - \eta_l}^{-1}= \frac{1}{\lambda}, \\
    &w_u\rbr{x} = \lambda\delta\rbr{{\eta_u > \ell\rbr{x; \tau, \beta}}}\rbr{\eta_u -  \ell\rbr{x; \tau, \beta}}^{-1}, \\
    &\sum_{i=1}^n \delta\rbr{{\eta_u >\ell\rbr{x; \tau, \beta}}}\rbr{\eta_u -  \ell\rbr{x; \tau, \beta}}^{-1}= \frac{1}{\lambda},
  \end{align*}
  where $\delta\rbr{a>b} = 
  \begin{cases}
  1 &\quad \text{if } a>b\\
  0 &\quad \text{otherwise}
  \end{cases}
  $.
  This is obtained by plugging the $f\rbr{x} = -\log x$ into~\eqref{eq:opt_weights} and considering $w^\top \one =1$, $w\ge 0$ and KKT conditions on the dual variables for $w\ge 0$. Unfortunately the reverse KL-divergence does not satisfy the conditions in Assumption \ref{asmp:f_div}. Note that this is the standard f-divergence function for empirical likelihood maximization problem, we therefore also include it here for the sake of completeness.
  
   \item {\bf $\chi^2$-divergence.} Notice that the standard f-divergence function, i.e., $f\rbr{x} = (x-1)^2$, of $\chi^2$-divergence $\operatorname{\chi^2}\rbr{w||\phat_n}:=\mathbb E_{\phat_n}\left[\left(\frac{w}{\phat_n}-1\right)^2\right]$ satisfies the conditions in Assumption \ref{asmp:f_div}. Consider the lower bound calculation. Leveraging the closed-form solution of the following $\ell_2$ projection problem onto the simplex space $w^\top \one =1$ and $w\ge 0$ \citep{wang2013projection}:
   \[
   \begin{split}
  &\arg \min_{w: w^\top \one =1, w\ge 0} \sum_{i=1}^nw_i\frac{\ell\rbr{x_i; \tau, \beta}}{\lambda}+ \sum_{i=1}^n \frac{1}{\phat_{n,i}}\left(w_i-\phat_{n,i}\right)^2 \\
   =&\sqrt{\phat_{n,i}}\cdot\arg  \min_{v: v^\top \sqrt{\phat_{n}} =1, v\ge 0}  \sum_{i=1}^n \left(v_i-(1-\frac{\ell\rbr{x_i; \tau, \beta}}{2\lambda})\cdot\sqrt{\phat_{n,i}}\right)^2, \,\,(\text{here we let } v_i=\frac{w_i}{\sqrt{\phat_{n,i}}})
   \end{split}
   \]
  the lower bound $w_\ell(x)$ is given by (for any $i \in \{1,2,\ldots,n\}$)
   \[
   \begin{split}
   w_\ell(x_i)&=\sqrt{\phat_{n,i}}\cdot w^*(x_i) \\
   &=\sqrt{\phat_{n,i}}\cdot \left((1-\frac{\ell\rbr{x_i; \tau, \beta}}{2\lambda})\cdot\sqrt{\phat_{n,i}}+\mathcal G_{\phat_{n}}\left((1-\frac{\ell\rbr{x; \tau, \beta}}{2\lambda})\cdot\sqrt{\phat_{n,i}}\right)\right)_+,
   \end{split}
   \]
   where $\mathcal G_{\phat_{n}}(y)=\frac{1 - \sum_{i=1}^{|\mathcal S_{\phat_{n}}|}y_i\cdot\sqrt{\phat_{n,i}}}{\sum_{i=1}^{|\mathcal S_{\phat_{n}}|}\phat_{n,i}}$, $\mathcal S_{\phat_{n}}$ is the set of indices in $\{1,\ldots,n\}$ in which any element $j$ satisfies $y_{(j)}+\frac{1}{\sum_{i=1}^j\phat_{n,i}}(1-\sum_{i=1}^j y_{(i)}\cdot\sqrt{\phat_{n,i}})>0$. Here $y_{(i)}$ indicates the samples with the $i$-th largest element of $y$.
Using analogous arguments, by replacing $\ell$ with $-\ell$ one can also define a similar solution for the upper bound $w_u(x)$. Now suppose $\phat_{n,i}=\frac{1}{n}$, $\forall i$. Then, we have
\begin{align*}
w_{l}(x_i) &=\sqrt{\frac{1}{n}}\cdot\left((1-\frac{\ell\rbr{x_i; \tau, \beta}}{2\lambda})\cdot\sqrt{\frac{1}{n}}+\mathcal G_{\frac{1}{n}}\left((1-\frac{\ell\rbr{x; \tau, \beta}}{2\lambda})\cdot\sqrt{\frac{1}{n}}\right)\right)_+, \\
w_{u}(x_i) &=\sqrt{\frac{1}{n}}\cdot\left((1+\frac{\ell\rbr{x_i; \tau, \beta}}{2\lambda})\cdot\sqrt{\frac{1}{n}}+\mathcal G_{\frac{1}{n}}\left((1+\frac{\ell\rbr{x; \tau, \beta}}{2\lambda})\cdot\sqrt{\frac{1}{n}}\right)\right)_+ ,
\end{align*}
where $\mathcal G_{\frac{1}{n}}(y)=\frac{n - \sum_{i=1}^{|\mathcal S_{1/n}|}y_i\cdot\sqrt{n}}{|\mathcal{S}_{1/n}|}$, $\mathcal S_{\frac{1}{n}}$ is the set of indices in $\{1,\ldots,n\}$ in which any element $j$ satisfies $y_{(j)}+\frac{1}{j}(n-\sqrt{n}\sum_{i=1}^j y_{(i)})>0$. Here $y_{(i)}$ indicates the samples with the $i$-th largest element of $y$.
This can also be verified by plugging the $f\rbr{x} = (x-1)^2$ into~\eqref{eq:opt_weights} and considering $w^\top \one =1$ and $w\ge 0$.
In fact, the above can be generalized to the Cressie-Read family with $f\rbr{x} = \frac{(x-1)^k - k(x-1) + k-1}{k\rbr{k-1}}$. 

  \item {\bf Reverse $KL$-divergence.} With the $f$-divergence function $f\rbr{x} = -\log x$ for the reverse $KL$-divergence, one has the following upper and lower bounds:
  \begin{align*}
    &w_l\rbr{x} = \lambda\delta\rbr{{ \ell\rbr{x; \tau, \beta} > \eta_l  }}\rbr{\ell\rbr{x; \tau, \beta} - \eta_l}^{-1}, \\
    &\sum_{i=1}^n \delta\rbr{{\ell\rbr{x; \tau, \beta} > \eta_l }}\rbr{\ell\rbr{x; \tau, \beta} - \eta_l}^{-1}= \frac{1}{\lambda}, \\
    &w_u\rbr{x} = \lambda\delta\rbr{{\eta_u > \ell\rbr{x; \tau, \beta}}}\rbr{\eta_u -  \ell\rbr{x; \tau, \beta}}^{-1}, \\
    &\sum_{i=1}^n \delta\rbr{{\eta_u >\ell\rbr{x; \tau, \beta}}}\rbr{\eta_u -  \ell\rbr{x; \tau, \beta}}^{-1}= \frac{1}{\lambda},
  \end{align*}
  where $\delta\rbr{a>b} = 
  \begin{cases}
  1 &\quad \text{if } a>b\\
  0 &\quad \text{otherwise}
  \end{cases}
  $.
  This is obtained by plugging the $f\rbr{x} = -\log x$ into~\eqref{eq:opt_weights} and considering $w^\top \one =1$, $w\ge 0$ and KKT conditions on the dual variables for $w\ge 0$. Unfortunately the reverse $KL$-divergence does not satisfy the conditions in~\asmpref{asmp:f_div}. Note that this is the standard $f$-divergence used in the vanilla empirical likelihood, we therefore also include it here for the sake of completeness. 
\end{itemize}

\subsection{Practical Algorithm}\label{appendix:practical_algorithm}

In~\eqref{eq:simple_upper_lower}, we eliminates one level optimization, and thus reduces the computational difficulty. Meanwhile, the SGDA for~\eqref{eq:simple_upper_lower} could benefit from the attractive finite-step convergence. However, as observed in~\citet{NamDuc16}, when $\lambda$ approaches $0$, the SGDA for~\eqref{eq:simple_upper_lower} may suffer from high variance. In this section, we consider two optional strategies to bypass such difficulty. We take the upper bound as an example, which can be applied for lower bound similarly:
\begin{itemize} 
	\item Instead of using the optimal weights~\eqref{eq:opt_weights}, \citet{NamDuc16} suggests to keep the $\rbr{w, \lambda}$ in optimization to be updated simultaneously via gradients, \ie, targeting on solving the Lagrangian~\eqref{eq:upper_step_III} with SGDA directly. For example, with $KL$-divergence, this leads to the update of $w_u$ in $t$-th iteration as
	\begin{equation}\label{eq:w_grad}
	\wtil^{(j)} = \exp\rbr{\eta_t\ell^{(j)} }\rbr{w^{(j)}}^{1 - \eta_t \lambda}\rbr{\frac{1}{n}}^{\eta_t\lambda}\quad \text{ and }\quad w_u = \frac{\wtil^{(j)}}{\sum_{j}\wtil^{(j)}},
	\end{equation}
	with $\eta_t$ as the stepsize. 

	\item The instability and high variance of solving~\eqref{eq:simple_upper_lower} comes from unboundness of $w$ induced by arbitarry $\lambda$ during the optimization procedure. In other words, given a fixed $\rbr{\tau, \beta}$, if we can keep $w\in\Kcal_f$ satisfied, \ie, 
	\begin{eqnarray}
	w_u = \argmax_{KL\rbr{w||\phat_n}\le \frac{\xi}{n}}\,\, \inner{w}{\ell} \quad \Rightarrow \quad \rbr{w_u, \lambda^*} = \argmax_{w^\top\one=1, w\ge 0}\argmin_{\lambda\ge 0} \,\, \inner{w}{\ell} - \lambda\rbr{KL\rbr{w||\phat_n} - \frac{\xi}{n}}\\
	\Rightarrow \rbr{w_u, {\lambda^*}}=\cbr{\tilde{w}_{\lambda^*}^{(j)}\defeq \exp\{ \ell^{(j)} /\lambda^* \}; \quad w_{\lambda^*}^{(j)} \defeq \tilde{w}_{\lambda^*}^{(j)} / \sum \tilde{w}_{\lambda^*}^{(j)}\quad \text{with}\quad \sum_{j=1}^n w_{\lambda^*}^{(j)}\log w_{\lambda^*}^{(j)} = \xi / n}, \label{eq:w_opt}
	\end{eqnarray}
	the optimization will be stable. 

\end{itemize}
Moreover, for the computational cost consideration, the major bottleneck in the optimization is updating the $w$, which is an $\Ocal\rbr{n}$ operation. Therefore, we leave the update of $w$ less frequently, which is corresponding to optimizing the equivalent form~\eqref{eq:upper_lower_outer}. Combined with these techniques into SGDA, we illustrate the algorithm in~\algref{alg:coindice}.

\begin{algorithm}[t]
   \caption{\estname: estimating upper confidence bound using KL-divergence and function approximation.}
\begin{algorithmic}
    \label{alg:coindice}
   \STATE {\bf Inputs}: A target policy $\pi$, a desired confidence $1-\alpha$, a finite sample dataset $\mathcal{D}\defeq\{(s_0^{(j)}, a_0^{(j)}, s^{(j)}, a^{(j)}, r^{(j)}, s^{\prime(j)})\}_{j=1}^n$, optimizers $\mathcal{OPT}_\theta$, number of iterations $K, T$.
   \vspace{2mm}
   \STATE Set divergence limit $\xi\defeq \frac{1}{2}\chi_1^{2,1-\alpha}$. 
   \STATE Initialize $\lambda\in\RR$, $Q_{\theta_1}: S\times A\to \mathbb{R}$, $\zeta_{\theta_2}: S\times A \to \mathbb{R}$.
   \FOR{$k=1,\dots, K$}
   \FOR{$t=1,\dots,T$}
   \STATE Sample from target policy $a_0^{(j)}\sim\pi(s_0^{(j)})$, $a^{(j)\prime}\sim\pi(s^{(j)\prime})$ for $j=1,\dots,n$.
   \STATE Compute loss terms: 
   \STATE $\ell^{(j)}\defeq (1-\gamma) Q_{\theta_1}(s_0^{(j)},a_0^{(j)}) + \zeta_{\theta_2}(s^{(j)}, a^{(j)}) \cdot(-Q_{\theta_1}(s^{(j)},a^{(j)}) + r^{(j)} + \gamma Q_{\theta_1}(s^{(j)\prime}, a^{(j)\prime}))$
   \STATE Update $(\theta_1,\theta_2)\leftarrow \mathcal{OPT}_\theta(\mathcal{L}, \theta_1,\theta_2)$.
   \ENDFOR
   \STATE Update $\rbr{w, \lambda}$ by~\eqref{eq:w_grad} or~\eqref{eq:w_opt}
   \STATE Compute loss $\mathcal{L} \defeq \sum_{j=1}^n w^{(j)}\cdot \ell^{(j)}$.
   \ENDFOR
   \vspace{2mm}
   \STATE {\bf Return} $\mathcal{L}$.
\end{algorithmic}
\end{algorithm}

\paragraph{Remark (More regularization for stability):} Directly solving a Lagrangian for LP may induce some instability, due to lack of curvature. To overcome such difficulty, the augmented Lagrangian method~(ALM)~\citep{Rockafellar74} is the natural choice. Directly applying the ALM will introduce the regularization $h\rbr{\EE_{\phat_n}\sbr{\Delta\rbr{x; \tau, \phi}}}$ where $h$ denotes some convex function with minimum at zero. Such regularization will not change the optimal solution $\rbr{\tau, \beta}$ in~\eqref{eq:upper_lower} and the value $[l_n, u_n]$. 

The ALM introduces extra computational cost in optimization since the regularization involves empirical expectations inside a nonlinear function. We exploit alternative regularizations following the spirit of ALM, while circumventing the computational difficulty. Recall the fact that the regularization on dual variable does not change the optimal solution~\citep[Theorem 4]{NacDaiKosChoetal19}, \ie
\begin{eqnarray}\label{eq:optimal_sol}
\tau^*\rbr{s, a} &=& \cbr{\argmax_{\tau\ge 0}\,\, \EE_{d^\Dcal}\sbr{\tau\cdot r\rbr{s, a}}\big| \EE_{d^\Dcal}\sbr{\Delta\rbr{x; \tau, \phi}} = 0}\\
&=& \cbr{\argmax_{\tau\ge 0}\,\, \EE_{d^\Dcal}\sbr{\tau\cdot r\rbr{s, a}} - \alpha \EE_{p}\sbr{h\rbr{\tau}} \big| \EE_{d^\Dcal}\sbr{\Delta\rbr{x; \tau, \phi}} = 0},
\end{eqnarray}
where $p$ is some distribution over $S\times A$. 

We show the upper bound as an example, and the lower bound can be treated similarly.  We have
\begin{eqnarray}\label{eq:dual_reg}
\rbr{w_u, \tau^*} &=& \argmax_{w\in \Kcal_f}\cbr{\argmax_{\tau\ge 0}\,\, \EE_{w}\sbr{\tau\cdot r\rbr{s, a}}\big| \EE_{w}\sbr{\Delta\rbr{x; \tau, \phi}} = 0}\nonumber \\
&=&\argmax_{w\in \Kcal_f}\cbr{\argmax_{\tau\ge 0}\,\, \EE_{w}\sbr{\tau\cdot r\rbr{s, a}} - \alpha \EE_{p}\sbr{h\rbr{\tau}} \big| \EE_{w}\sbr{\Delta\rbr{x; \tau, \phi}} = 0},
\end{eqnarray}
where the equality comes from~\citet[Theorem 4]{NacDaiKosChoetal19} and the fact the regularization  $\EE_{p}\sbr{h\rbr{\tau}}$ does not depend on $w$. Then, we can solve~\eqref{eq:dual_reg} alternatively for $\rbr{w_u, \tau^*}$ by Lagrangian,
\begin{equation}\label{eq:dual_reg_lagrangian}
\max_{\tau\ge 0}\min_{\beta}\max_{w\in \Kcal_f}\,\, \EE_{w}\sbr{\tau\cdot r\rbr{s, a} + \beta^\top \Delta\rbr{x; \tau, \phi}} - \alpha \EE_{p}\sbr{h\rbr{\tau}}.
\end{equation}
Although the optimal $\tilde\beta^*$ to~\eqref{eq:dual_reg_lagrangian} differs from $\beta^*$, $\rbr{w_u, \tau^*}$ are the same. Once we have the $\rbr{w_u, \tau^*}$, we can recover the original Lagrangian $\tilde\rho_\pi\rbr{w_u} = \EE_{w_u}\sbr{\tau\cdot r\rbr{s, a}}$, since $\EE_{w_u}\sbr{\beta^{*\top} \Delta\rbr{x; \tau^*, \phi}} = 0$ in the original Lagrangian $\EE_w\rbr{\ell\rbr{x; \tau^*, \beta^*}}$ in~\eqref{eq:upper_lower} due to the KKT condition. 

Comparing to the original ALM, the new regularization takes the advantage of ALM while keeps the original computational efficiency. 

\section{Proofs for Statistical Properties}\label{appendix:proofs_theorems}

In this section, we provide the detailed proofs for the asymptotic coverage~\thmref{thm:asymptotic_coverage} and the finite-sample correction~\thmref{thm:finite_sample}. For notation simplicity, we use $\sup$, $\max$ and ${\inf, \min}$ interchangeably. With a little abuse of notation, we use $\int$ as $\sum$ on discrete domain.

\subsection{Asymptotic Coverage}\label{appendix:asymptotic}

\thmref{thm:asymptotic_coverage} follows from a result in~\citet{DucGlyNam16}.  The following notation will be needed:
\begin{itemize}
  \item $\ell\rbr{x; \tau, \beta}= \rbr{1 - \gamma} \beta^\top \phi\rbr{s_0, a_0} + \tau\rbr{s, a}\rbr{r\rbr{s, a} + \gamma \beta^\top\phi\rbr{s', a'} - \beta^\top \phi\rbr{s, a}}$;
  \item $\nbr{f}_1 \defeq \int \abr{f\rbr{s, a}} d^\Dcal\rbr{s, a}dsda$, and $\nbr{\phi\rbr{s, a}}_2 \defeq \sqrt{\inner{\phi}{\phi}}$;
  \item $\nbr{f\rbr{s, a}}_{L^2\rbr{d^\Dcal}}\defeq \EE_{d^\Dcal}\sbr{f^2\rbr{s, a}}^{\frac{1}{2}}$, $\Hcal\subset L^2\rbr{d^\Dcal}$,
    we define $L^{\infty}\rbr{\Hcal}$ be the space of bounded linear functionals on $\Hcal$ with $\nbr{L_1 - L_2}_{\Hcal}\defeq \sup_{h\in\Hcal}|L_1h - L_2h|$ for $L_1, L_2\in L^\infty\rbr{\Hcal}$;
  \item $p = \frac{dP}{d\mu}$, with a Lebesgue measure $\mu$, is the Radon-Nikodym derivative. Abusing notation a bit, we use $\rbr{D_f\rbr{P||Q}, D\rbr{p||q}}$, and $\rbr{\EE_P\sbr{\cdot}, \EE_p\rbr{\cdot}}$ interchangeably. 
\end{itemize}

\begin{definition}\citep[{\bf Hadamard directionally differentiability}]{DucGlyNam16}\label{def:hadamard_diff}
Let $\Qcal$ be the space of signed measures bounded with norm $\nbr{\cdot}_\Hcal$. The functional $T: \Pcal\rightarrow \RR$ is \emph{Hadamard directionally differentiable} at $P\in \Pcal$ tangentially to $B\subset \Qcal$ if for all $H\in B$, there exists $dT_p\rbr{H}\in\RR$ such that for all convergent sequences $t_n\rightarrow 0$ and $\nbr{H_n - H}_\Hcal\rightarrow 0$ that satisfies $P + t_nH_n\in \Pcal$, the following holds
$$
\frac{T\rbr{P + t_n H_n} - T\rbr{P}}{t_n}\rightarrow dT_P\rbr{H},\quad \text{as}\quad n\rightarrow \infty. 
$$
We say $T:\Pcal \rightarrow \RR$ has an \emph{influence function} $T^1\rbr{x; P}\in\RR$ if
$$
dT_P\rbr{Q - P}\defeq \int T^1\rbr{x; P}d\rbr{Q-P}\rbr{x},
$$
and $\EE_P\sbr{T^{1}\rbr{x; P}} = 0$. 
\end{definition}

We consider $f$ in $D_f$ satisfying the following assumption~\citep{DucGlyNam16},
\vspace{-2mm}
\begin{assumption}[Smoothness of $f$-divergence]\label{asmp:f_div}
The function $f: \RR_+\rightarrow \RR$ is convex, three times differentiable in a neighborhood of $1$, $f\rbr{1} = f'\rbr{1} = 0$ and $f''\rbr{1} = 2$.\footnote{That $f(1)=0$ is required in the definition of f-divergence.  If $f'\rbr{1}\ne0$, one can ``lift'' it by $\bar{f}\rbr{t}=f(t) - f'(1)(t-1)$ so that the new function satisfies $\bar{f}'\rbr{1}=0$. $f''\rbr{1} = 2$ is assumed for easier calculation without loss of generality, as discussed in~\citet{DucGlyNam16}. For example, one can use $f\rbr{t} = 2x\log x- 2\rbr{x-1}$ for modified KL-divergence, $f\rbr{t} = \rbr{x-1}^2$ for $\chi^2$-divergence, and $f\rbr{t} = -\log x +(x-1)-\frac{1}{2}\rbr{x-1}^2$ for reverse KL-divergence. 
}
\end{assumption}
Then, the following theorem, which slightly simplifies~\citet[Theorem 10]{DucGlyNam16}, characterizes the asymptotic coverage of the general uncertainty estimation,
\begin{theorem}[General asymptotic coverage]\label{thm:general}
Let~\asmpref{asmp:f_div} hold and $\Hcal = \cbr{h\rbr{x; \tau, \beta}}$, where $h\rbr{x; \tau, \beta}$ is Lipschitz and the space of $\rbr{\tau, \beta}$ is compact. Denote $B\subset\Qcal$ be such that $\nbr{\sqrt{n}\rbr{\Phat_n - P_0} - G}_\Hcal\rightarrow 0$ with $G\in B$. Assume $T:\Pcal\rightarrow \RR$ is Hadamard differentiable at $P_0$ tangentially to $B$ with influence function $T^1\rbr{\cdot; P_0}$ and $dT_P$ is defined and continuous on the whole $\Qcal$, then,  
$$
\lim_{n\rightarrow \infty}\PP\rbr{T\rbr{P_0}\in \cbr{T\rbr{P}: D_f\rbr{P||P_n}\le \frac{\xi}{n}}} = \PP\rbr{\chi_1^2\le\xi}. 
$$
\end{theorem}

Denote the $T\rbr{P} = \max_{\tau\ge 0}\min_{\beta\in \RR^p} \EE_P\sbr{\ell\rbr{x; \tau, \beta}}$ by convexity-concavity, our proof for~\thmref{thm:asymptotic_coverage} will be mainly checking the conditions required by~\thmref{thm:general}: {\bf i)}, Lipschitz continuity of functions in $\Hcal$, and {\bf ii)} Hadamard differentiability of $T\rbr{P}$.  

We first specify the regularity assumption for stationary distribution ratio:
\begin{assumption}[Stationary ratio regularity]\label{asmp:bounded_ratio}
The target stationary state-action correction rato is bounded: $\nbr{\tau^*}_\infty\le C_\tau<\infty$, 
and $\tau^*\in\Fcal_\tau$ where $\Fcal_\tau$ is a convex, compact and bounded RKHS space with bounded kernel function $\nbr{k\rbr{\rbr{\cdot, \cdot}, \rbr{s, a}}}_{\Fcal_\tau}\le K$.
\end{assumption}
The bounded ratio component of~\asmpref{asmp:bounded_ratio} is a standard assumption used in~\citet{NacChoDaiLi19,ZhaDaiLiSch20,UehHuaJia19}. The latter part regarding $\Fcal_\tau$ is required for the existence of solutions. In fact, the RKHS assumption $\Fcal_\tau$ is already quite flexible, and it includes deep neural networks by adopting the neural tangent kernels~\citep{AroDuHuLietal19}. 

With~\asmpref{asmp:bounded_ratio}, we can immediately obtain 
$$
T\rbr{P} = \max_{\tau\in\Fcal_\tau}\min_{\beta\in \RR^p} \EE_P\sbr{\ell\rbr{x; \tau, \beta}} = \min_{\beta\in \RR^p}\max_{\tau\in\Fcal_\tau} \EE_P\sbr{\ell\rbr{x; \tau, \beta}}
$$
by the minimax theorem~\citep[Proposition 2.1]{EkeTem99}. By this equivalence, we will focus on the $\min$-$\max$ form. 

Since $r \in\sbr{0, \rmax}$, one has for every $\pi$ that $Q^\pi \le \rmax/(1-\gamma)$. Therefore, it is reasonable to assume the following regularity conditions for $\phi$:
\begin{assumption}[Embedding feature regularity]\label{asmp:phi_regularity}
There exist some finite constants $C_\beta$ and $C_\phi$, such that $\nbr{\beta}_2\le C_\beta$, $\nbr{\phi}_2\le C_\phi$.  Moreover, $\phi\rbr{s, a}$ is $L_\phi$-Lipschitz continuous. 
\end{assumption}
This assumption implies $\nbr{\beta^\top\phi}_\infty\le \nbr{\beta}_2\nbr{\phi}_2\le C_\beta C_\phi$ and Lipschitz continuity of $\beta^\top\phi\rbr{s, a}$. We define $\Fcal_\beta\defeq \cbr{\beta| \nbr{\beta}_2\le C_\beta}$. 

\begin{lemma}[Lipschitz continuity]\label{lem:lip_cont}
Under \asmpsref{asmp:bounded_ratio} and~\ref{asmp:phi_regularity}, 
function $\ell$ satisfies $\nbr{\ell\rbr{x; \tau, \beta}}_\infty\le M$ and is $C_\ell$-Lipschitz in $\rbr{\tau, \beta}$, for some proper finite constants $M$ and $C_\ell$.
\end{lemma}
\begin{proof}
We first show the boundedness claim.  By the definition of $\ell\rbr{x; \tau, \beta}$, one has
\begin{eqnarray*}
\lefteqn{\nbr{\ell\rbr{x; \tau, \beta}}_\infty} \\
&\le& \rbr{1 - \gamma}\nbr{\beta^\top\phi}_\infty + \nbr{\tau\rbr{s, a}\rbr{r\rbr{s, a} + \gamma \beta^\top\phi\rbr{s', a'} - \beta^\top \phi\rbr{s, a}}}_\infty \\
&\le &
\rbr{1 - \gamma}\nbr{\beta^\top\phi}_\infty + \nbr{\tau\rbr{s, a}}_{\infty} \rbr{r\rbr{s, a} + \gamma {\beta^\top\phi\rbr{s', a'} - \beta^\top \phi\rbr{s, a}}}\\
&\le&
\rbr{1 - \gamma} C_\beta C_\phi + C_\tau \rbr{R_{\max} + \rbr{1 + \gamma}C_\beta C_\phi} \\
& = &\rbr{C_\tau + 1}\rbr{1 - \gamma} C_\beta C_\phi + C_\tau R_{\max} \defeq M.
\end{eqnarray*}

We equip $\Fcal_\tau\times\Fcal_\beta$ with the norm 
\begin{align}
\nbr{\rbr{\tau, \beta}} \defeq \nbr{\tau}_{\Fcal_\tau} + \nbr{\beta}_{2},
\end{align}
Then, we show the Lipschitz continuity of $\ell\rbr{x; \tau, \beta}$ in $\rbr{\tau, \beta}$,
\begin{eqnarray*}
\lefteqn{\abr{\ell\rbr{x; \tau_1, \beta_1} - \ell\rbr{x; \tau_2, \beta_2}}} \\
&\le&  \rbr{1 - \gamma}\abr{\phi\rbr{s_0, a_0}^\top\rbr{\beta_1 - \beta_2}} + \abr{\tau_2\rbr{s, a}\rbr{\beta_1 - \beta_2}^\top \rbr{\gamma\phi\rbr{s', a'} + \phi\rbr{s, a}}}\\
&+& \abr{\rbr{\tau_1\rbr{s, a} - \tau_2\rbr{s, a}} \rbr{r\rbr{s, a} + \gamma \beta_1^\top\phi\rbr{s', a'} - \beta_1^\top\phi\rbr{s, a}}}\\
&\le&\rbr{1 - \gamma} \rbr{\rbr{2 + \gamma}C_\phi + C_\tau}\nbr{\beta_1 - \beta_2}_2 + \rbr{R_{\max} + \rbr{1 + \gamma}C_\phi C_\beta}\abr{\tau_1\rbr{s, a} - \tau_2\rbr{s, a}},\\
&\le&\rbr{1 - \gamma} \rbr{\rbr{2 + \gamma}C_\phi + C_\tau}\nbr{\beta_1 - \beta_2}_2 + \rbr{R_{\max} + \rbr{1 + \gamma}C_\phi C_\beta}K\nbr{\tau_1 - \tau_2}_{\Fcal_\tau},\\
&\le& C_\ell \rbr{\nbr{\beta_1 - \beta_2}_2 + \nbr{\tau_1 - \tau_2}_{\Fcal_\tau}},
\end{eqnarray*}
which implies the $\ell\rbr{x; \tau, \beta}$ is $C_\ell$-Lipschitz continuous with 
$$
C_\ell\defeq \max\cbr{\rbr{1 - \gamma} \rbr{\rbr{2 + \gamma}C_\phi + C_\tau, \rbr{1 + \gamma}C_\phi C_\beta}K }.
$$ 
\end{proof}
We now check the Hadamard directional differentiability of $T\rbr{P}$. The following proof largely follows~\citet{DucGlyNam16,Romisch14}. 
\begin{lemma}[Hadamard Differentiability]\label{lem:hadamard_diff}
Under~\asmpsref{asmp:bounded_ratio} and~\ref{asmp:phi_regularity}, the functional $T\rbr{P} = \min_{\beta\in \Fcal_\beta}\max_{\tau\in\Fcal_\tau} \EE_P\sbr{\ell\rbr{x; \tau, \beta}}$ is Hadamard directionally differentiable on $\Pcal$ tangentially to $B\rbr{\Hcal, P_0}\subset L^\infty\rbr{\Hcal}$ with derivative
\begin{equation*}
dT_P\rbr{H}\defeq \int \ell\rbr{x; \tau^*, \beta^*}dH\rbr{x},
\end{equation*}
where $\rbr{\beta^*, \tau^*} = \argmin_{\beta\in \Fcal_\beta}\argmax_{\tau\in\Fcal_\tau} \EE_{P_0}\sbr{\ell\rbr{x; \tau, \beta}}$. 
\end{lemma}
\begin{proof}
For convenience, we define
\begin{equation*}
\Htil\rbr{\tau, \beta}\defeq \int \ell\rbr{x; \tau, \beta}dH\rbr{x},
\end{equation*}
where $H$ is associated with a measure in $\Qcal$. 

We first show the upper bound convergence. For $H_n\in B\rbr{\Hcal, P_0}$ with $\nbr{H_n - H}_\Hcal\rightarrow 0$, for any sequence $t_n\rightarrow 0$, we have
\begin{eqnarray*}
\lefteqn{T\rbr{P_0 + t_nH_n} - T\rbr{P_0}} \\
&=& \min_{\beta\in \Fcal_\beta}\max_{\tau\in\Fcal_\tau} \rbr{\EE_{P_0}\sbr{\ell\rbr{x; \tau, \beta}} + t_n\Htil_n\rbr{\tau, \beta} } - \min_{\beta\in \Fcal_\beta}\max_{\tau\in\Fcal_\tau} \EE_{P_0}\sbr{\ell\rbr{x; \tau, \beta}}\\
&\le& \max_{\tau\in\Fcal_\tau}\rbr{\EE_{P_0}\sbr{\ell\rbr{x; \tau, \beta^*}} + t_n\Htil_n\rbr{\tau, \beta^*}} - \EE_{P_0}\sbr{\ell\rbr{x; \tau, \beta^*}}\\
&\le& \max_{\tau\in\Fcal_\tau} t_n\Htil_n\rbr{\tau, \beta^*}.
\end{eqnarray*}
Denote $\tau_n^* = \argmax_{\tau\in\Fcal_\tau}\Htil_n\rbr{\tau, \beta^*}$, by definition, we have
\begin{eqnarray*}
\max_{\tau\in\Fcal_\tau} \Htil_n\rbr{\tau, \beta^*} - \max_{\tau\in\Fcal_\tau} \Htil\rbr{\tau, \beta^*} \le \Htil_n\rbr{\tau_n^*, \beta^*} - \Htil\rbr{\tau_n^*, \beta^*} 
\le \nbr{\Htil_n - \Htil}_\Hcal \rightarrow 0.
\end{eqnarray*}
Therefore, we obtain
\begin{equation*}
\limsup_{n}\frac{1}{t_n} \rbr{T\rbr{P_0 + t_nH_n} - T\rbr{P_0}} \le \Htil\rbr{\tau^*, \beta^*}. 
\end{equation*}

For the lower bound part, we have
\begin{eqnarray*}
  \lefteqn{T\rbr{P_0 + t_nH_n}} \\
&=& \min_{\beta \in \Fcal_\beta} \cbr{ \max_{\tau\in\Fcal_\tau} \rbr{\EE_{P_0}\sbr{\ell\rbr{x; \tau, \beta}} +t_n\Htil_n\rbr{\tau, \beta} } }\\
&=& \min_{\beta\in \Fcal_\beta} \cbr{ {\EE_{P_0}\sbr{\ell\rbr{x; \tau_n\rbr{\beta}, \beta}} +t_n\rbr{\Htil_n\rbr{\tau_n\rbr{\beta}, \beta} -\Htil\rbr{\tau_n\rbr{\beta}, \beta} } } +  t_n\Htil\rbr{\tau_n\rbr{\beta}, \beta} } \\
&\le & \min_{\beta\in \Fcal_\beta} \cbr{ {\EE_{P_0}\sbr{\ell\rbr{x; \tau_n\rbr{\beta}, \beta}} +t_n\nbr{\Htil_n -\Htil}_\Hcal }  + t_n\nbr{\Htil}_\Hcal } \\
&\le&\min_{\beta\in \Fcal_\beta} { {\EE_{P_0}\sbr{\ell\rbr{x; \tau_n\rbr{\beta}, \beta}} } }+ \Ocal\rbr{1}\cdot t_n,
\end{eqnarray*}
where $\tau_n\rbr{\beta} = \argmax_{\tau\in\Fcal_\tau} \rbr{\EE_{P_0}\sbr{\ell\rbr{x; \tau, \beta}} +t_n\Htil_n\rbr{\tau, \beta} }$.

Denote the set of $\epsilon$-ball of solutions w.r.t. $P$ as
\begin{equation*}
S_{P}\rbr{\epsilon}\defeq \cbr{{\beta\in \Fcal_\beta}: \max_{\tau\in\Fcal_\tau}\EE_{P}\sbr{\ell\rbr{x; \tau, \beta}} \le  \min_{\beta\in \Fcal_\beta}\max_{\tau\in\Fcal_\tau}\EE_{P}\sbr{\ell\rbr{x; \tau, \beta}} + \epsilon}.
\end{equation*}
Then, $\beta^*_n \in S_{P_0+t_nH_n}\rbr{0}$ implies $\beta^*_n\in S_{P_0}\rbr{ct_n}$, which in turn implies the sequence of $\beta_n^*$ has a subsequence $\tilde\beta_m^*$ that converges to $\beta^*\in S_{P_0}\rbr{0}$. 

It is straightforward to check the Lipschitz continuity of $\bar\ell\rbr{\beta}\defeq \max_{\tau}\EE\sbr{\ell\rbr{x; \tau, \beta}}$ as
\begin{eqnarray*}
  \lefteqn{\abr{\bar\ell\rbr{\beta_1} - \bar\ell\rbr{\beta_2} }} \\
&\le& \rbr{1 - \gamma}\nbr{\beta_1 - \beta_2}_2\EE_{\mu_0\pi}\sbr{\nbr{\phi_{s_0, a_0}}}_2 + \abr{\max_{\tau\in\Fcal_\tau} \EE\sbr{\tau\cdot r + \beta_1^\top \Delta} - \max_{\tau\in\Fcal_\tau} \EE\sbr{\tau\cdot r + \beta_2^\top \Delta}}\\
&\le&\rbr{1 - \gamma}\nbr{\beta_1 - \beta_2}_2\EE_{\mu_0\pi}\sbr{\nbr{\phi_{s_0, a_0}}}_2 + \max_{\tau\in\Fcal} \abr{\EE\sbr{\tau\cdot r + \beta_1^\top \Delta} - \EE\sbr{\tau\cdot r + \beta_2^\top \Delta}}\\
&\le&\rbr{1 - \gamma}\nbr{\beta_1 - \beta_2}_2\EE_{\mu_0\pi}\sbr{\nbr{\phi_{s_0, a_0}}}_2 + \max_{\tau\in\Fcal} \abr{\EE\sbr{ \rbr{\beta_1 - \beta_2}^\top \Delta}}\\
&\le& \rbr{\rbr{1 - \gamma}C_\phi + C_\tau\rbr{1 + \gamma}C_\phi} \nbr{\beta_1 - \beta_2}_2.
\end{eqnarray*}
Therefore, with $\tilde\beta_n^*\rightarrow \beta^*$, we have
\begin{equation*}
\lim_m \tilde\ell\rbr{\tilde\beta_m^*} = \min_{\beta}\tilde\ell\rbr{\beta} = T\rbr{P_0},
\end{equation*}
and due to the optimality, for any $m$,
\begin{equation*}
\tilde\ell\rbr{\tilde\beta_m^*} \ge \min_{\beta}\tilde\ell\rbr{\beta}. 
\end{equation*}
\begin{eqnarray*}
  \lefteqn{T\rbr{P_0 + t_m H_m} - T\rbr{P_0}} \\
&\ge& \max_{\tau\in\Fcal_\tau}\cbr{\EE_{P_0}\sbr{\ell\rbr{x; \tau, \tilde\beta_m^*}}  + t_n \Htil_n\rbr{\tau, \tilde\beta_m^*}} - \max_{\tau\in\Fcal_\tau}\EE_{P_0}\sbr{\ell\rbr{x; \tau,\tilde\beta_m^*}}\\
&\ge& \EE_{P_0}\sbr{\ell\rbr{x; \tau_m\rbr{\tilde\beta_m^*}, \tilde\beta_m^*}}  + t_n \Htil_n\rbr{\tau_m\rbr{\tilde\beta_m^*}, \tilde\beta_m^*} - \EE_{P_0}\sbr{\ell\rbr{x; \tau_m\rbr{\tilde\beta_m^*}, \tilde\beta_m^*}}\\
&=&t_n \Htil_n\rbr{\tau_m\rbr{\tilde\beta_m^*}, \tilde\beta_m^*},
\end{eqnarray*}
where $\tau_m\rbr{\tilde\beta_m^*} = \argmax_{\tau\in\Fcal_\tau}\EE_{P_0}\sbr{\ell\rbr{x; \tau,\tilde\beta_m^*}}$.

Since $\tilde\beta_m^*\rightarrow \beta^*$, we have $\tau_m\rbr{\tilde\beta_m^*}\rightarrow \tau^*$, and thus,
\begin{eqnarray*}
&&\abr{\Htil_n\rbr{\tau_m\rbr{\tilde\beta_m^*}, \tilde\beta_m^*} - \Htil\rbr{\tau^*, \beta^*}}\\
&\le& \abr{\Htil_n\rbr{\tau_m\rbr{\tilde\beta_m^*}, \tilde\beta_m^*} - \Htil\rbr{\tau_m\rbr{\tilde\beta_m^*}, \tilde\beta_m^*}}  + \abr{\Htil\rbr{\tau_m\rbr{\tilde\beta_m^*}, \tilde\beta_m^*} - \Htil\rbr{\tau^*, \beta^*}}\\
&\le&\nbr{\Htil_n - \Htil}_\Hcal + \abr{\Htil\rbr{\tau_m\rbr{\tilde\beta_m^*}, \tilde\beta_m^*} - \Htil\rbr{\tau^*, \beta^*}}\rightarrow 0,
\end{eqnarray*}
where we use $\ell\rbr{\tau, \beta; x}$ is Lipschitz continuous. Therefore, we obtain
\begin{equation*}
\liminf_{n} \frac{1}{t_n} \rbr{T\rbr{P_0 + t_nH_n} - T\rbr{P_0}} \ge \Htil\rbr{\tau^*, \beta^*}.
\end{equation*}
\end{proof}

\repeatthm{thm:asymptotic_coverage}{\bf (Asymptotic coverage)}
{\thmasymptoticcoverage}

\begin{proof}
The proof is to verify the conditions in~\thmref{thm:general} hold. By~\lemref{lem:dice_region_II}, we can rewrite 
\begin{equation*}
\PP\rbr{\rho_\pi\in C^f_{n, \confrange}} = \PP\rbr{\rho_\pi\in \cbr{\rholag\rbr{w}\big|w\in\Kcal_f}},
\end{equation*}
where, according to the boundedness assumption on $\beta$ in Assumption~\ref{asmp:phi_regularity},
\begin{equation*}
\rholag\rbr{w} = \max_{\tau\ge 0}\min_{\beta\in \Fcal_\beta} \EE_w\sbr{\tau\cdot r + \beta^\top \Delta\rbr{x; \tau, \phi}} = \min_{\beta\in \Fcal_\beta}\max_{\tau\ge 0} \EE_w\sbr{\tau\cdot r + \beta^\top \Delta\rbr{x; \tau, \phi}}.
\end{equation*}
With~\lemref{lem:lip_cont} and~\lemref{lem:hadamard_diff}, the conditions in~\thmref{thm:general} are satisfied. We apply~\thmref{thm:general} on the unique optimal solution $\rbr{\tau^*, \beta^*} =\argmin_{\beta\in \Fcal_\beta}\argmax_{\tau\ge 0}\EE_{P_0}\sbr{\ell\rbr{x; \tau, \beta}}$. We have $dT_P$ is a linear functional on the space of bounded measures and
$$
dT_{P_0}\rbr{H} = \int \ell\rbr{x; \tau^*, \beta^*}dH\rbr{x},
$$
with the canonical gradient given by $T^1\rbr{\cdot; P_0} = \ell\rbr{x; \tau^*, \beta^*} - \EE_{P_0}\sbr{\ell\rbr{x; \tau^*, \beta^*}}$. 
\end{proof}

\subsection{Finite-Sample Correction}\label{appendix:finite_sample}

The previous section considers the asymptotic coverage of~\estname. We now analyze the finite-sample effect for the estimator, for the special case $f\rbr{x} = \rbr{x-1}^2$.  Thus, $D_f$ is the $\chi^2$-divergence.

Consider the optimization problem,
\begin{eqnarray}\label{eq:z_upper}
\max_{w\in\RR^n} w^\top z, \quad \st\quad D_f\rbr{w||\phat_n}\le \frac{\confrange}{n}, w\in \Pcal^{n-1}\rbr{\phat_n}.
\end{eqnarray}
The following result will be needed.
\begin{lemma}\citep[{\bf Theorem 1}]{NamDuc17}\label{lem:variance_upper}
Let $Z\in[0, M]$ be a random variable, $\sigma^2 = Var\rbr{Z}$ and $s_n^2 = \EE_{\phat_n}\sbr{Z^2} - \EE_{\phat_n}\sbr{Z}^2$ as the population and sample variance of $Z$, respectively. For $\confrange\ge0$, we have
\begin{equation*}
\sbr{\sqrt{\frac{\confrange}{n}s_n^2} - \frac{M\confrange}{n}}_+\le \max_w\cbr{\EE_w\sbr{Z}| D_f\rbr{w||\phat_n}\le\frac{\confrange}{n}, w\in\Pcal^{n-1}\rbr{\phat_n}} - \EE_{\phat_n}\sbr{Z}\le \sqrt{\frac{\confrange}{n}s_n^2}.
\end{equation*}
Moreover, for $n\ge \max\cbr{2, \frac{M^2}{\sigma^2}\max\cbr{4\sigma, 22}}$, with probability at least $1 - \exp\rbr{-\frac{3n\sigma^2}{5M^2}}$,
\begin{equation*}
\max_w\cbr{\EE_w\sbr{Z}| D_f\rbr{w||\phat_n}\le\frac{\confrange}{n}, w\in\Pcal^{n-1}\rbr{\phat_n}} = \EE_{\phat_n}\sbr{Z}+ \sqrt{\frac{\confrange}{n}s_n^2}.
\end{equation*}
\end{lemma}

The follow is the symmetric version of \lemref{lem:variance_upper}, which can be obtained immediately by negating the random variable $Z$.  For completeness, we give the proof below, which is adapted from \citet{{NamDuc17}}. Recall that the lower bound is obtained by solving the following:
\begin{eqnarray}\label{eq:z_lower}
\min_{w\in\RR^n} w^\top z, \quad \st\quad D_f\rbr{w||\phat_n}\le \frac{\confrange}{n}, \,\, w\in \Pcal^{n-1}\rbr{\phat_n}.
\end{eqnarray}

\begin{lemma}[Lower bound variance representation]\label{lem:variance_lower}
Under the same conditions in~\lemref{lem:variance_upper}, for $\confrange\ge0$, we have
\begin{equation*}
\sbr{\sqrt{\frac{\confrange}{n}s_n^2} - \frac{M\confrange}{n}}_+\le  \EE_{\phat_n}\sbr{Z} - \min_w\cbr{\EE_w\sbr{Z}| D_f\rbr{w||\phat_n}\le\frac{\confrange}{n}, w\in\Pcal^{n-1}\rbr{\phat_n}}\le \sqrt{\frac{\confrange}{n}s_n^2}.
\end{equation*}
Moreover, for $n\ge \max\cbr{2, \frac{M^2}{\sigma^2}\max\cbr{4\sigma, 22}}$, with probability at least $1 - \exp\rbr{-\frac{3n\sigma^2}{5M^2}}$,
\begin{equation*}
\min_w\cbr{\EE_w\sbr{Z}| D_f\rbr{w||\phat_n}\le\frac{\confrange}{n}, w\in\Pcal^{n-1}\rbr{\phat_n}} = \EE_{\phat_n}\sbr{Z} - \sqrt{\frac{\confrange}{n}s_n^2}.
\end{equation*}
\end{lemma}
\begin{proof}
Denote $u = \frac{1}{n} - w$, we have $u^\top\one = 0$, and the optimization~\eqref{eq:z_lower} can be written as
\begin{equation}
\zbar - \max_u u^\top\rbr{z - \zbar}, \quad \st\quad \nbr{u}_2^2\le \frac{\confrange}{n}, u^\top \one = 0, u\le \frac{1}{n},
\end{equation}
with $\zbar = \frac{1}{n}\sum_{i=1}^n z_i$. Obviously, by the Cauchy-Schwartz inequality, 
\begin{equation*}
u^\top\rbr{z - \zbar}\le \sqrt{\frac{\confrange}{n}}\nbr{z - \zbar}_2,
\end{equation*}
and the equality holds if and only if
\begin{equation*}
u_i = \frac{\sqrt{\confrange}\rbr{z - \zbar}}{n\nbr{z-\zbar}_2} = \frac{\sqrt{\confrange}\rbr{z - \zbar}}{n\sqrt{ns_n^2}}.
\end{equation*}
Given the constraint $u\le \frac{1}{n}$, to achieve the maximum, one needs to ensure
\begin{equation*}
\max_{i} \frac{\sqrt{\confrange}\rbr{z - \zbar}}{n\sqrt{ns_n^2}} \le 1. 
\end{equation*}
If this condition is satisfied, we have
\begin{equation*}
\EE_{\phat_n}\sbr{Z} - \min_w\cbr{\EE_w\rbr{Z}| D_f\rbr{w||\phat_n}\le\frac{\confrange}{n}
}\le \sqrt{\frac{\confrange}{n}s_n^2}.
\end{equation*}
Since $z\in[0, M]$, we have $\abr{z_i - \zbar}\le M$, to ensure the condition, we need
$
\frac{{\confrange}M^2}{ns_n^2}\le 1\Leftrightarrow s_n^2\ge\frac{\confrange M^2}{n}. 
$
Otherwise, suppose $s_n^2<\frac{\confrange M^2}{n}$, or equivalently $\frac{\confrange s_n^2}{n}< \frac{\confrange^2M^2}{n}$, then,
\begin{equation*}
\min_{w} w^\top z\le \EE_{\phat_n}\sbr{z} - \sbr{\sqrt{\frac{\confrange}{n}s_n^2} - \frac{M\confrange}{n}}_+.
\end{equation*}

For the high-probability statement, when $n\ge \max\cbr{2, \frac{M^2}{\sigma^2}\max\cbr{4\sigma, 22}}$, and the event $s_n^2\ge\frac{3}{64}\sigma^2$ holds, $s_n^2\ge\frac{\confrange M^2}{n}$. Following~\citet[Theorem~10]{maurer2009empirical}, one can bound that
\begin{equation*}
\PP\rbr{\abr{s_n - \sigma}\le a}\le \exp\rbr{-\frac{n a^2}{2M^2}}.
\end{equation*}
Setting $a = \rbr{1 - \frac{\sqrt{3}}{8}}\sigma$ completes the proof.
\end{proof}

With~\lemref{lem:variance_upper} and~\lemref{lem:variance_lower}, we represent the confidence bounds with variance. We resort to an empirical Bernstein bound applied to the function space $\Fcal$ with bounded function $h:\Xcal\rightarrow [0, M]$, using empirical $\ell_\infty$-covering numbers, $\Ncal_\infty\rbr{\Fcal, \epsilon, n}$,
\begin{lemma}\citep[{\bf Theorem 6}]{maurer2009empirical}\label{lem:union_variance}
Let $n\ge \frac{8M^2}{t}$ and $t\ge\log 12$. Then, with probability at least $1 - 6\Ncal_\infty\rbr{\Fcal, \epsilon, 2n}e^{-t}$, for any $h\in\Fcal$,
\begin{equation*}
{\EE\sbr{h} - \EE_{\phat_n}\sbr{h}}\le \sqrt{\frac{18Var_{\phat_n}\rbr{h}t}{n}} + \frac{15Mt}{n} + 2\rbr{1 + 2\sqrt{\frac{2t}{n}}}\epsilon.
\end{equation*}
\end{lemma}

\newcommand{\thmfinitesample}{
Denote by $\Ncal_\infty\rbr{\Fcal_\tau, \epsilon, 2n}$ and $\Ncal_\infty\rbr{\Fcal_\beta, \epsilon, 2n}$ the $\ell_\infty$-covering numbers of $\Fcal_\tau$ and $\Fcal_\beta$ with $\epsilon$-ball on $2n$ empirical samples, respectively.  Let $D_f$ be $\chi^2$-divergence. Under~\asmpsref{asmp:bounded_ratio} and~\ref{asmp:phi_regularity}, let $M\defeq\rbr{C_\tau + 1}\rbr{1 - \gamma} C_\beta C_\phi + C_\tau R_{\max} $ and $C_\ell\defeq \max\cbr{\rbr{1 - \gamma} \rbr{\rbr{2 + \gamma}C_\phi + C_\tau, \rbr{1 + \gamma}C_\phi C_\beta}K } $, then, we have
\begin{equation*}
\PP\rbr{\rho_\pi\in \sbr{l_u - \zeta_n, u_n + \zeta_n}}\ge 1 - 12\Ncal_\infty\rbr{\Fcal_\tau, {\epsilon}, 2n}\Ncal_\infty\rbr{\Fcal_\beta, {\epsilon}, 2n}e^{-\frac{\confrange}{18}},
\end{equation*}
where $\rbr{l_n, u_n}$ are the solutions to~\eqref{eq:upper_lower}, $\zeta_n = \frac{11M\confrange}{6n} + 2\rbr{1 + 2\sqrt{\frac{\confrange}{9n}}}C_\ell\epsilon$ and $\confrange = \chi_{(1)}^{2, 1-\alpha}$.

When the \texttt{VC}-dimensions of $\Fcal_\tau$ and $\Fcal_\beta$ (denoted by $d_{\Fcal_\phi}$
and $d_{\Fcal_\beta}$, respectively) are finite, we have
$$
\PP\rbr{\rho_\pi\in [l_n-\kappa_n, u_n + \kappa_n]} \ge 1 - 12\exp\rbr{c_1 + 2\rbr{d_{\Fcal_\tau} + d_{\Fcal_\beta} -1}\log n - \frac{\confrange}{18}},
$$
where $c_1 = 2c + \log d_{\Fcal_\tau} + \log d_{\Fcal_\beta} + \rbr{d_{\Fcal_\tau}  + d_{\Fcal_\beta} -1}$, and $\kappa_n = \frac{11M\confrange}{6n} + 2\frac{C_\ell M}{n}\rbr{1 + 2\sqrt{\frac{\confrange}{9n}}}.$
}
\noindent
\repeatthm{thm:finite_sample}{\bf (Finite-sample correction)}
\textit{\thmfinitesample}

\begin{proof}
We focus on the upper bound, and the lower bound can be bounded in a similar way.  Define
\begin{align*}
\rbr{\tau^*, \beta^*} &\defeq \argmax_{\tau\in \Fcal_\tau}\argmin_{\beta} \EE_{d^\Dcal}\sbr{\ell\rbr{x; \tau, \beta}}\\
\rbr{\what^*, \hat\tau^*, \hat\beta^*} &\defeq\argmax_w\argmax_{\tau\in \Fcal_\tau}\argmin_{\beta} \EE_{w}\sbr{\ell\rbr{x; \tau, \beta}}.
\end{align*} 

By definition and the optimality of $\beta^*$, we have
\begin{eqnarray}\label{eq:upper_step1}
\rho_\pi = \EE_{d^\Dcal}\sbr{\ell\rbr{x; \tau^*, \beta^*}}\le \EE_{d^\Dcal}\sbr{\ell\rbr{x; \tau^*, \hat\beta^*}}.
\end{eqnarray}

Applying~\lemref{lem:union_variance} and the Lipschitz-continuity of $\ell\rbr{x; \tau, \beta}$ on $\Fcal_\tau\times \Fcal_\beta$, with probability at least $1 - 6\Ncal_\infty\rbr{\Fcal_\tau, {\epsilon}, 2n}\Ncal_\infty\rbr{\Fcal_\beta, {\epsilon}, 2n}e^{-t}$, we have
\begin{eqnarray*}
  \lefteqn{\EE_{d^\Dcal}\sbr{\ell\rbr{x; \tau^*, \hat\beta^*}}} \\
&\le& \EE_{\phat_n}\sbr{\ell\rbr{x; \tau^*, \hat\beta^*}} + 3\sqrt{\frac{2Var_{\phat_n}\rbr{\ell\rbr{x; \tau^*, \hat\beta^*}}t}{n}} + \frac{15Mt}{n} + 2\rbr{1 + 2\sqrt{\frac{2t}{n}}}C_\ell\epsilon\\
&\le& \max_{D_f\rbr{w||\phat_n}\le\frac{\confrange}{n}}\EE_w\sbr{\ell\rbr{x; \tau^*, \hat\beta^*}} -\sbr{\sqrt{\frac{\confrange Var_{\phat_n}\rbr{\ell\rbr{x; \tau^*, \hat\beta^*}}}{n}} - \frac{M\confrange}{n}}_+ \\
&&  + 3\sqrt{\frac{2Var_{\phat_n}\rbr{\ell\rbr{x; \tau^*, \hat\beta^*}}t}{n}} + \frac{15Mt}{n} + 2\rbr{1 + 2\sqrt{\frac{2t}{n}}}C_\ell\epsilon\\
&\le&\max_{D_f\rbr{w||\phat_n}\le\frac{\confrange}{n}}\max_{\tau\in\Fcal_\tau}\min_{\beta\in \Fcal_\beta}\EE_w\sbr{\ell\rbr{x; \tau, \beta}} +\frac{11}{6n}M\confrange  + 2\rbr{1 + 2\sqrt{\frac{2t}{n}}}C_\ell\epsilon
\end{eqnarray*}
where the second equation comes from~\lemref{lem:variance_upper} and the third line comes from setting $t\le \frac{\confrange}{18}$ and the definition of $\hat\beta^*$. Combining this with~\eqref{eq:upper_step1}, we may conclude that with probability at least $1 - 6\Ncal_\infty\rbr{\Fcal_\tau, {\epsilon}, 2n}\Ncal_\infty\rbr{\Fcal_\beta, {\epsilon}, 2n}e^{-\frac{\confrange}{18}}$, 
\begin{eqnarray*}
\rho_\pi \le\max_{D_f\rbr{w||\phat_n}\le\frac{\confrange}{n}}\max_{\tau\in\Fcal_\tau}\min_{\beta\in \Fcal_\beta}\EE_w\sbr{\ell\rbr{x; \tau, \beta}} + \frac{11M\confrange}{6n} + 2\rbr{1 + 2\sqrt{\frac{\confrange}{9n}}}C_\ell\epsilon.
\end{eqnarray*}

With the same strategy based on~\lemref{lem:variance_lower} and~\lemref{lem:union_variance}, we can also bound the finite-sample lower bound correction that with probability at least $1 - 6\Ncal_\infty\rbr{\Fcal_\tau, {\epsilon}, 2n}\Ncal_\infty\rbr{\Fcal_\beta, {\epsilon}, 2n}e^{-\frac{\confrange}{18}}$, 
\begin{eqnarray*}
\rho_\pi \ge\max_{D_f\rbr{w||\phat_n}\le\frac{\confrange}{n}}\max_{\tau\in\Fcal_\tau}\min_{\beta\in \Fcal_\beta}\EE_w\sbr{\ell\rbr{x; \tau, \beta}} - \frac{11M\confrange}{6n} - 2\rbr{1 + 2\sqrt{\frac{\confrange}{9n}}}C_\ell\epsilon.
\end{eqnarray*}
The first part of the theorem is then proved.

For the second part, by~\citet[Theorem 2.6.7]{VaaWel96}, one can bound $\Ncal\rbr{\Fcal, \epsilon, 2n}\le c\texttt{VC}\rbr{\Fcal}\rbr{\frac{16Mne}{\epsilon}}^{\texttt{VC}\rbr{\Fcal} - 1}$ for some constant $c$. We set $\epsilon = \frac{M}{n}$ and denote $d_\Fcal = \texttt{VC}\rbr{\Fcal}$. Plugging this into the bound, we obtain 
$$
\PP\rbr{\rho_\pi\in [l_n-\kappa, u_n + \kappa]} \ge 1 - 12\exp\rbr{c_1 + 2\rbr{d_{\Fcal_\tau} + d_{\Fcal_\beta} -1}\log n - \frac{\confrange}{18}},
$$
where $c_1$ and $\kappa$ are as given in the theorem statement.
\end{proof}

\section{Implementing Principles of Optimism and Pessimism}\label{appendix:opt_pes}

Based on the discussion in~\secref{sec:opt_pes}, the optimism and pessimism principles can be implemented by maximizing $u_\Dcal\rbr{\pi}$ and $l_\Dcal\rbr{\pi}$, respectively. In this section, we first calculate the gradient $\nabla_\pi u_\Dcal\rbr{\pi}$ and $\nabla_\pi l_\Dcal\rbr{\pi}$, and elaborate on the algorithm details. 

Since we will optimize the policy $\pi$, we modify the confidence interval estimator in~\estname slightly, so that $\pi$ is an explicitly parameterized distribution.
Concretely, we consider the samples $x\defeq \rbr{s_0, s, a, r}$ with $s_0\sim \mu_0\rbr{s}$, $\rbr{s, a, r, s'}\sim d^\Dcal$, which leads to the corresponding upper and lower bounds with 
\[
\tilde\ell\rbr{x; \tau, \beta, \pi}\defeq \tau\cdot r + \beta^\top \tilde\Delta\rbr{x; \tau, \phi, \pi}, 
\]
where $\tilde\Delta\rbr{x; \tau, \phi, \pi} = \rbr{1 - \gamma}\EE_{\pi\rbr{a_0|s_0}}\sbr{\phi\rbr{s_0, a_0}} + \gamma \EE_{\pi\rbr{a'|s'}}\sbr{\phi\rbr{s', a'}\tau\rbr{s, a}} - \phi\rbr{s, a}\tau\rbr{s, a}$.

\begin{theorem}\label{thm:grad_computation}
Given optimal $(\beta_l^*, \tau_l^*, w_l^*)$ and $(\beta_u^*, \tau_u^*, w_u^*)$ for lower and upper bounds, respectively, the gradients of $l_\Dcal\rbr{\pi}$ and $u_\Dcal\rbr{\pi}$ can be computed as
\begin{equation}
\begin{bmatrix}
\nabla_\pi l_\Dcal\rbr{\pi} \\
\\
\\
\nabla_\pi u_\Dcal\rbr{\pi}
\end{bmatrix}
=\begin{bmatrix}
\EE_{w_l^*}\sbr{\rbr{1 - \gamma}\EE_{a_0\sim\pi}\sbr{\nabla_\pi \log \pi\rbr{a_0|s_0}{\beta_l^*}^\top\phi\rbr{s_0, a_0}} + \right.\\
\qquad\qquad\qquad\qquad \left. \gamma \EE_{a'\sim \pi\rbr{a'|s'}}\sbr{\tau_l^*\rbr{s, a}\nabla_\pi\log \pi\rbr{a'|s'}{{\beta_l^*}^\top \phi\rbr{s', a'}}}} \\
\EE_{w_u^*}\sbr{\rbr{1 - \gamma}\EE_{a_0\sim\pi}\sbr{\nabla_\pi \log \pi\rbr{a_0|s_0}{\beta_u^*}^\top\phi\rbr{s_0, a_0}} + \right. \\
\qquad\qquad\qquad\qquad \left. \gamma \EE_{a'\sim \pi\rbr{a'|s'}}\sbr{\tau_u^*\rbr{s, a}\nabla_\pi\log \pi\rbr{a'|s'}{{\beta_u^*}^\top \phi\rbr{s', a'}}}}
\end{bmatrix}
\end{equation}
\end{theorem}

\begin{proof}
We focus on the computation of $\nabla_\pi u_\Dcal\rbr{\pi}$ with the optimal $(\beta_u^*, \tau_u^*, w_u^*)$:
\begin{eqnarray}
&&\nabla_\pi u_\Dcal\rbr{\pi} = \EE_{w_u^*}\sbr{\nabla_\pi \tilde\ell\rbr{x; \tau, \beta}} \nonumber \\
&=& \rbr{1 - \gamma}
\EE_{w_u^*}\nabla_\pi\EE_{a_0\sim\pi}\sbr{{\beta_u^*}^\top\phi\rbr{s_0, a_0}} + \gamma \EE_{w_u^*}\sbr{\tau_u^*\rbr{s, a}
\nabla_\pi\EE_{a'\sim \pi\rbr{a'|s'}}\sbr{{\beta_u^*}^\top \phi\rbr{s', a'}}}\nonumber \\
&=& \rbr{1 - \gamma}\EE_{w_u^*}\EE_{a_0\sim\pi}\sbr{\nabla_\pi \log \pi\rbr{a_0|s_0}{\beta_u^*}^\top\phi\rbr{s_0, a_0}} \\
& & + \gamma \EE_{w_u^*}\EE_{a'\sim \pi\rbr{a'|s'}}\sbr{\tau_u^*\rbr{s, a}\nabla_\pi\log \pi\rbr{a'|s'}{{\beta_u^*}^\top \phi\rbr{s', a'}}}.\label{eq:upper_grad}
\end{eqnarray}
The case for the lower bound can be obtained similarly:
\begin{align}\label{eq:lower_grad}
\nabla_\pi l_\Dcal\rbr{\pi} =& \rbr{1 - \gamma}\EE_{w_l^*}\EE_{a_0\sim\pi}\sbr{\nabla_\pi \log \pi\rbr{a_0|s_0}{\beta_l^*}^\top\phi\rbr{s_0, a_0}} \\
& \,\,\,\, + \gamma \EE_{w_l^*}\EE_{a'\sim \pi\rbr{a'|s'}}\sbr{\tau_l^*\rbr{s, a}\nabla_\pi\log \pi\rbr{a'|s'}{{\beta_l^*}^\top \phi\rbr{s', a'}}}. \nonumber
\end{align}

\end{proof}

Now, we are ready to apply the policy gradient upon $u_\Dcal\rbr{\pi}$ or $l_\Dcal\rbr{\pi}$ to implement the optimism for exploration or pessimism for safe policy improvement, respectively. We illustrate the details in~\algref{alg:opt_pes}. 

\begin{algorithm}[H]
   \caption{\estname-OPT: implementation of optimism/pessimism principle}
\begin{algorithmic}
    \label{alg:opt_pes}
   \STATE {\bf Inputs}: initial policy $\pi_0$, a desired confidence $1-\alpha$, a finite sample dataset $\mathcal{D}\defeq\{x^{(j)}=(s_0^{(j)}, s^{(j)}, a^{(j)}, r^{(j)}, s^{\prime(j)})\}_{j=1}^n$, number of iterations $T$.
   \FOR{$t=1,\dots,T$}
	   \STATE Estimate $\rbr{\beta_u^*, \tau_u^*, w_u^*}$ via~\algref{alg:coindice} for optimism. \COMMENT{$\rbr{\beta_l^*, \tau_l^*, w_l^*}$ for pessimism.} 
	   \STATE Sample $\cbr{x^{(j)}}_{j=1}^k\sim \Dcal_t$, $a_0^{(j)}\sim\pi_t(s_0^{(j)})$, $a^{(j)\prime}\sim\pi_t(s^{(j)\prime})$ for $j=1,\dots,k$.
	   \STATE Estimate the stochastic approximation to $\nabla_\pi u_{\Dcal_t}\rbr{\pi_t}$ via~\eqref{eq:upper_grad}.  \COMMENT{$\nabla_\pi l_{\Dcal_t}\rbr{\pi_t}$ via~\eqref{eq:lower_grad} for pessimism.}
	   \STATE Natural policy gradient update: $\pi_{t+1} = \argmin_{\pi} -\inner{\pi}{\nabla_\pi u_{\Dcal_t}\rbr{\pi_t}} +  \frac{1}{\eta}KL\rbr{\pi||\pi_t}$. \\
	   \COMMENT{$\pi_{t+1} = \argmin_{\pi} -\inner{\pi}{\nabla_\pi l_{\Dcal_t}\rbr{\pi_t}} +  \frac{1}{\eta}KL\rbr{\pi||\pi_t}$ for pessimism.} 
	   \STATE Collect samples $\Ecal = \cbr{x^{(j)} =\rbr{s_0, s, a, r, s'}^{(j)}}_{j=1}^m$ by executing $\pi_{t+1}$, $\Dcal_{t+1} = \Dcal_t\cup\Ecal$. \\
	   \COMMENT{Skip the data collection step in offline setting.}
   \ENDFOR
   \STATE {\bf Return} $\pi_{T}$.
\end{algorithmic}
\end{algorithm}

\end{appendix}

\end{document}